\documentclass[a4paper,10pt]{article}
\usepackage{amsmath,amssymb}
\usepackage{appendix}
\usepackage{setspace}
\usepackage{graphicx}
\usepackage{varioref}           %   smart cross-referencing
\usepackage{makeidx}            %   make a index
\usepackage{amsthm}
\usepackage{colortbl}
\usepackage{indentfirst}
\usepackage{varioref}
\usepackage{float}
\usepackage{subfigure}
\usepackage{enumerate}
\usepackage{booktabs}
\usepackage{geometry}
\usepackage{xcolor}
\usepackage{algorithm}          % format of the algorithm
\usepackage{algorithmic}        % format of the algorithm
\usepackage{multirow}           % multirow for format of table
\usepackage[small]{caption}
\usepackage{url}
\usepackage{authblk}            % for authorship
\usepackage[pagebackref=true,
            breaklinks=true,
            letterpaper=true,
            colorlinks=true,
            citecolor=blue,
            bookmarks=false]{hyperref}

\newtheorem{theorem}{Theorem}[section]
\newtheorem{lemma}[theorem]{Lemma}

\newtheorem{remark}{Remark}[section]

\title{\textbf{Sparse Kernel Canonical Correlation Analysis via $\ell_1$-regularization}\footnote{Part of the material in this paper was presented in~\cite{Chu13} and~\cite{ZhangThesis13}.}}
\author[1]{Xiaowei~Zhang\thanks{Corresponding author: \href{mailto:zxwtroy87@gmail.com}{zxwtroy87@gmail.com}}}
\author[1]{Delin~Chu}
\author[2]{Li-Zhi~Liao}
\author[2]{Michael~K.~Ng}
\affil[1]{Department of Mathematics, National University of Singapore.}
\affil[2]{Department of Mathematics, Hong Kong Baptist University.}

\date{}

\begin{document}
\maketitle

\begin{abstract}
Canonical correlation analysis (CCA) is a multivariate statistical technique for finding the linear relationship between two sets of variables. The kernel generalization of CCA named kernel CCA has been proposed to find nonlinear relations between datasets. Despite their wide usage, they have one common limitation that is the lack of sparsity in their solution. In this paper, we consider sparse kernel CCA and propose a novel sparse kernel CCA algorithm (SKCCA). Our algorithm is based on a relationship between kernel CCA and least squares. Sparsity of the dual transformations is introduced by penalizing the $\ell_{1}$-norm of dual vectors. Experiments demonstrate that our algorithm not only performs well in computing sparse dual transformations but also can alleviate the over-fitting problem of kernel CCA.
\end{abstract}

\section{Introduction}
The description of relationship between two sets of variables has long been an interesting topic to many researchers. Canonical
correlation analysis (CCA), which was originally introduced in \cite{Hotelling36}, is a multivariate statistical technique for
finding the linear relationship between two sets of variables. Those two sets of variables can be considered as different views
of the same object or views of different objects, and are assumed to contain some joint information in the correlations between
them. CCA seeks a linear transformation for each of the two sets of variables in a way that the projected variables in the
transformed space are maximally correlated.

Let $\{x_{i}\}^{n}_{i=1} \in \mathbf{R}^{d_{1}}$ and $\{y_{i}\}^{n}_{i=1} \in \mathbf{R}^{d_{2}}$ be $n$ samples for variables
$x$ and $y$, respectively. Denote
$$X=
\begin{bmatrix}
      x_{1} & \cdots & x_{n} \\
     \end{bmatrix} \in \mathbf{R}^{d_{1}\times n},  \quad
Y = \begin{bmatrix}
      y_{1} & \cdots & y_{n} \\
     \end{bmatrix} \in \mathbf{R}^{d_{2}\times n}, $$
and assume both $\{x_{i}\}^{n}_{i=1}$ and $\{y_{i}\}^{n}_{i=1}$ have zero mean, i.e., $\sum\limits^{n}_{i=1}x_{i} = 0$ and
$\sum\limits^{n}_{i=1}y_{i} = 0$. Then CCA solves the following optimization problem
\begin{equation}\label{singleCCA}
    \begin{array}{rl}
      \max\limits_{w_{x},w_{y}} & w^{T}_{x}XY^{T}w_{y} \\
      s.t.               & w^{T}_{x}XX^{T}w_{x} = 1, \\
                         & w^{T}_{y}YY^{T}w_{y} = 1,
    \end{array}
\end{equation}
to get  the first pair of \textit{weight vectors} $w_{x}$ and $w_{y}$, which are further utilized to obtain the first pair of
\textit{canonical variables} $w^{T}_{x}X$ and $w^{T}_{y}$Y, respectively. For the rest pairs of weight vectors and canonical
variables, CCA solves sequentially the same problem as \eqref{singleCCA} with additional constraints of orthogonality among
canonical variables. Suppose we have obtained a pair of linear transformations $W_{x} \in \mathbf{R}^{d_{1} \times l}$ and $W_{y}
\in \mathbf{R}^{d_{2} \times l}$, then for a pair of new data $(x, y)$, its projection into the new coordinate system determined
by $(W_{x}, W_{y})$ will be
\begin{equation}\label{CCAproj}
    (W^{T}_{x}x, W^{T}_{y}y).
\end{equation}

Since CCA only consider linear transformation of the original variables, it can not capture nonlinear relations among variables.
However, in a wide range of practical problems linear relations may not be adequate for studying relation among variables.
Detecting nonlinear relations among data is important and useful in modern data analysis, especially when dealing with data that
are not in the form of vectors, such as text documents, images, micro-array data and so on. A natural extension, therefore, is to
explore and exploit nonlinear relations among data. There has been a wide concern in the nonlinear CCA \cite{Dauxois98,Lai01},
among which one most frequently used approach is the kernel generalization of CCA, named kernel canonical correlation analysis
(kernel CCA). Motivated from the development and successful applications of kernel learning methods
\cite{Scholkopf02,Shawe-Taylor04}, such as support vector machines (SVM) \cite{Burges98,Scholkopf02}, kernel principal component
analysis (KPCA) \cite{Scholkopf98}, kernel Fisher discriminant analysis \cite{Mika99}, kernel partial least squares
\cite{Rosipal01} and so on, there has emerged lots of research on kernel CCA
\cite{Akaho01,Melzer01,Bach03,Fukumizu07,Fyfe00,Hardoon04,Hardoon09,Kuss03,Lai01,Shawe-Taylor04}.

Kernel methods have attracted a great deal of attention in the field of nonlinear data analysis. In kernel methods, we first
implicitly represent data as elements in reproducing kernel Hilbert spaces associated with positive definite kernels, then apply
linear algorithms on the data and substitute the linear inner product by kernel functions, which results in nonlinear variants.
The main idea of kernel CCA is that we first virtually map data $X$ into a high dimensional \textit{feature space}
$\mathcal{H}_{x}$ via a mapping $\phi_{x}$ such that data in the feature space become
$$\Phi_{x}=
\begin{bmatrix}
      \phi_{x}(x_{1}) & \cdots & \phi_{x}(x_{n}) \\
     \end{bmatrix} \in \mathbf{R}^{\mathcal{N}_{x} \times n},
$$
where $\mathcal{N}_{x}$ is the dimension of feature space $\mathcal{H}_{x}$ that can be very high or even infinite. The mapping
$\phi_{x}$ from input data to the feature space $\mathcal{H}_{x}$ is performed implicitly by considering a \textit{positive
definite kernel function} $\kappa_{x}$ satisfying
\begin{equation}\label{kernelfunction}
    \kappa_{x}(x_{1},x_{2})=\langle \phi_{x}(x_{1}),\phi_{x}(x_{2}) \rangle,
\end{equation}
where $\langle \cdot,\cdot \rangle$ is an inner product in $\mathcal{H}_{x}$, rather than by giving the coordinates of
$\phi_{x}(x)$ explicitly. The feature space $\mathcal{H}_{x}$ is known as the \textit{Reproducing Kernel Hilbert Space (RKHS)}
\cite{Wahba99} associated with kernel function $\kappa_{x}$. In the same way, we can map $Y$ into a feature space
$\mathcal{H}_{y}$ associated with kernel $\kappa_{y}$ through mapping $\phi_{y}$ such that
$$
\Phi_{y} = \begin{bmatrix}
      \phi_{y}(y_{1}) & \cdots & \phi_{y}(y_{n}) \\
     \end{bmatrix} \in \mathbf{R}^{\mathcal{N}_{y} \times n}.
$$
After mapping $X$ to $\Phi_{x}$ and $Y$ to $\Phi_{y}$, we then apply ordinary linear CCA to data pair $(\Phi_{x}, \Phi_{y})$.

Let
\begin{equation}\label{GramMatrix}
    K_{x} = \langle \Phi_{x},\Phi_{x} \rangle = [\kappa_{x}(x_{i},x_{j})]^{n}_{i,j=1} \in \mathbf{R}^{n\times n}, \quad
    K_{y} = \langle \Phi_{y},\Phi_{y} \rangle = [\kappa_{y}(y_{i},y_{j})]^{n}_{i,j=1} \in \mathbf{R}^{n\times n}
\end{equation}
be matrices consisting of inner products of datasets $\mathcal{X}$ and $\mathcal{Y}$, respectively. $K_{x}$ and $K_{y}$ are
called \textit{kernel matrices} or \textit{Gram matrices}. Then kernel CCA seeks linear transformation in the feature space by
expressing the weight vectors as linear combinations of the training data, that is
\begin{equation*}
    w_{x} = \Phi_{x}\alpha = \sum\limits^{n}_{i=1}\alpha_{i}\phi_{x}(x_{i}), \quad
    w_{y} = \Phi_{y}\beta = \sum\limits^{n}_{i=1}\beta_{i}\phi_{y}(y_{i}),
\end{equation*}
where $\alpha,~\beta \in \mathbf{R}^{n}$ are called \textit{dual vectors}. The first pair of dual vectors can be determined by
solving the following optimization problem
\begin{equation}\label{singleKCCA}
    \begin{array}{rl}
      \max\limits_{\alpha,\beta} & \alpha^{T}K_{x}K_{y}\beta \\
      s.t. & \alpha^{T}K^{2}_{x}\alpha = 1, \\
           & \beta^{T}K^{2}_{y}\beta = 1.
    \end{array}
\end{equation}
The rest pairs of dual vectors are obtained via sequentially solving the same problem as \eqref{singleKCCA} with extra
constraints of orthogonality. More details on the derivation of kernel CCA are presented in Section \ref{background}.

Suppose we have obtained dual transformations $\mathcal{W}_{x},~\mathcal{W}_{y} \in \mathbf{R}^{n \times l}$ and corresponding
CCA transformations $W_{x}\in \mathbf{R}^{\mathcal{N}_{x} \times l}$ and $W_{y}\in \mathbf{R}^{\mathcal{N}_{y} \times l}$ in
feature spaces, then projection of data pair $(x,y)$ onto the kernel CCA directions can be computed by first mapping $x$ and $y$
into the feature space $\mathcal{H}_{x}$ and $\mathcal{H}_{y}$, then evaluate their inner products with $W_{x}$ and $W_{y}$. More
specifically, projections can be carried out as
\begin{equation}\label{projx}
\langle W_{x}, \phi_{x}(x) \rangle = \langle \Phi_{x}\mathcal{W}_{x}, \phi_{x}(x) \rangle = \mathcal{W}^{T}_{x}K_{x}(X,x),
\end{equation}
with $K_{x}(X,x)= \begin{bmatrix}
                            \kappa_{x}(x_{1},x) & \cdots & \kappa_{x}(x_{n},x) \\
                          \end{bmatrix}^{T}$,
and
\begin{equation}\label{projy}
\langle W_{y}, \phi_{y}(y) \rangle = \langle \Phi_{y}\mathcal{W}_{y}, \phi_{y}(y) \rangle = \mathcal{W}^{T}_{y}K_{y}(Y,y),
\end{equation}
with $K_{y}(Y,y)= \begin{bmatrix}
                            \kappa_{y}(y_{1},y) & \cdots & \kappa_{y}(y_{n},y) \\
                          \end{bmatrix}^{T}$.

Both optimization problems \eqref{singleCCA} and \eqref{singleKCCA} can be solved by considering generalized eigenvalue problems
\cite{Bie05} of the form
\begin{equation}\label{GEVP}
Ax=\lambda Bx,
\end{equation}
where $A$, $B$ are symmetric positive semi-definite. This generalized eigenvalue problem can be solved efficiently using
approaches from numerical linear algebra \cite{Golub96}. CCA and kernel CCA have been successfully applied in many fields,
including cross$-$language documents retrieval \cite{Vinokourov03}, content$-$based image retrieval \cite{Hardoon04},
bioinformatics \cite{Vert03,Yamanishi03}, independent component analysis \cite{Bach03,Fyfe00}, computation of principal angles
between linear subspaces \cite{Bjorck73,Golub94}.

Despite the wide usage of CCA and kernel CCA, they have one common limitation that is lack of sparseness in transformation
matrices $W_{x}$ and $W_{y}$ and dual transformation matrices $\mathcal{W}_{x}$ and $\mathcal{W}_{y}$. Equation \eqref{CCAproj}
shows that projections of the data pair $x$ and $y$ are linear combinations of themselves which make interpretation of the
extracted features difficult if the transformation matrices $W_{x}$ and $W_{y}$ are dense. Similarly, from \eqref{projx} and
\eqref{projy} we can see that the kernel functions $\kappa_{x}(x_{i},x)$ and $\kappa_{y}(y_{i},y)$ must be evaluated for all
$\{x_{i}\}^{n}_{i=1}$ and $\{y_{i}\}^{n}_{i=1}$ when dual transformation matrices $\mathcal{W}_{x}$ and $\mathcal{W}_{y}$ are
dense, which can lead to excessive computational time to compute projections of new data. To handle the limitation of CCA,
researchers suggested to incorporate sparsity into weight vectors and many papers have studied sparse CCA
\cite{ChuLNZ13,Hardoon11,Parkhomenko09,Sriperumbudur07,Sriperumbudur11,Waaijenborg08,Wiesel08,Witten09,WittenHastie09}. Similarly,
we shall find sparse solutions for kernel CCA so that projections of new data can be computed by evaluating the kernel function
at a subset of the training data. Although there are many sparse kernel approaches \cite{Bishop06}, such as support vector
machines \cite{Scholkopf02}, relevance vector machine \cite{Tipping01} and sparse kernel partial least squares
\cite{DhanjalGunn08,Momma03}, seldom can be found in the area of sparse kernel CCA \cite{Dhanjal08,TanFyfe01}.

In this paper we first consider a new sparse CCA approach and then generalize it to incorporate sparsity into kernel CCA. A
relationship between CCA and least squares is established so that CCA solutions can be obtained by solving a least squares
problem. We attempt to introduce sparsity by penalizing $\ell_{1}$-norm of the solutions, which eventually leads to a
$\ell_{1}$-norm penalized least squares optimization problem of the form
\begin{equation*}
    \min_{x \in \mathbf{R}^{d}}\frac{1}{2}\|Ax-b\|^{2}_{2} + \lambda \|x\|_{1},
\end{equation*}
where $\lambda > 0$ is a regularizer controlling the sparsity of $x$. We adopt a fixed-point continuation (FPC) method
\cite{Hale08,Hale10} to solve the $\ell_{1}$-norm regularized least squares above, which results in a new sparse CCA algorithm
(SCCA$\_$LS). Since the optimization criteria of CCA and kernel CCA are of the same form, the same idea can be extended to kernel
CCA to get a sparse kernel CCA algorithm (SKCCA).

The remainder of the paper is organized as follows. In Section \ref{background}, we present background results on both CCA and
kernel CCA, including a full parameterization of the general solutions of CCA and a detailed derivation of kernel CCA. In Section
\ref{LSCCA}, we first establish a relationship between CCA and least squares problems, then based on this relationship we propose
to incorporate sparsity into CCA by penalizing the least squares with $\ell_{1}$-norm. Solving the penalized least squares
problems by FPC leads to a new sparse CCA algorithm SCCA$\_$LS. In Section \ref{KernelCCA}, we extend the idea of deriving
SCCA$\_$LS to its kernel counterpart, which results in a novel sparse kernel CCA algorithm SKCCA. Numerical results of applying
the newly proposed algorithms to various applications and comparative empirical results with other algorithms are presented in
Section \ref{experiments}. Finally, we draw some conclusion remarks in Section \ref{conclusions}.

\section{Background}\label{background}
\setcounter{equation}{0}

In this section we provide enough background results on CCA and kernel CCA so as to make the paper self-contained. In the first
subsection, we present the full parameterization of the general solutions of CCA and related results; in the second subsection,
based on the parameterization in previous subsection, we demonstrate a detailed derivation of kernel CCA.

\subsection{Canonical correlation analysis}

As stated in Introduction, by solving \eqref{singleCCA}, or equivalently
\begin{equation}\label{singleCCAequi}
    \begin{array}{rl}
      \min\limits_{w_{x},w_{y}} & \|X^{T}w_{x} - Y^{T}w_{y}\|_{2}^{2} \\
      s.t.               & w^{T}_{x}XX^{T}w_{x} = 1, \\
                         & w^{T}_{y}YY^{T}w_{y} = 1,
    \end{array}
\end{equation}
we can get a pair of weight vectors $w_{x}$ and $w_{y}$ for CCA. Only one pair of weight vectors is not enough for most practical
problems, however. To obtain multiple projections of CCA, we recursively solve the following optimization problem
\begin{equation}\label{recursiveCCA}
    \begin{array}{rl}
      (w_{x}^{k},w_{y}^{k}) = \mbox{arg}\max\limits_{w_{x},w_{y}} & w_{x}^{T}XY^{T}w_{y} \\
                                               s.t.               & w^{T}_{x}XX^{T}w_{x} = 1, \\
                                                                  & X^{T}w_{x} \perp \{X^{T}w_{x}^{1},\cdots,X^{T}w_{x}^{k-1}\}, \\
                                                                  & w^{T}_{y}YY^{T}w_{y} = 1, \\
                                                                  & Y^{T}w_{y} \perp \{Y^{T}w_{y}^{1},\cdots,Y^{T}w_{y}^{k-1}\}, \\
    \end{array}\quad
    k = 2,\cdots, l,
\end{equation}
where $l$ is the number of projections we need. The unit vectors $X^{T}w_{x}^{k}$ and $Y^{T}w_{y}^{k}$ in \eqref{recursiveCCA}
are called the $k$th pair of \textit{canonical variables}. If we denote
$$
W_{x} =
\begin{bmatrix}
      w_{x}^{1} & \cdots & w_{x}^{l} \\
\end{bmatrix} \in \mathbf{R}^{d_1\times l}, \quad
W_{y} =
\begin{bmatrix}
      w_{y}^{1} & \cdots & w_{y}^{l} \\
\end{bmatrix} \in \mathbf{R}^{d_{2}\times l},
$$
then we can show \cite{ChuLNZ13} that the optimization problem above is equivalent to
\begin{equation}\label{CompactCCA}
    \begin{array}{rl}
      \max\limits_{W_{x},W_{y}} & \mbox{Trace}(W^{T}_{x}XY^{T}W_{y}) \\
      s.t.               & W^{T}_{x}XX^{T}W_{x} = I,~W_{x}\in \mathbf{R}^{d_{1}\times l}, \\
                         & W^{T}_{y}YY^{T}W_{y} = I,~W_{y}\in \mathbf{R}^{d_{2}\times l}.
    \end{array}
\end{equation}
Hence, optimization problem \eqref{CompactCCA} will be used as the criterion of CCA.

A solution of \eqref{CompactCCA} can be obtained via solving a generalized eigenvalue problem of the form
\eqref{GeneralSolutionThm}. Furthermore, we can fully characterize all solutions of the optimization problem \eqref{CompactCCA}.
Define
$$
r = \mbox{rank}(X), \quad s= \mbox{rank}(Y), \quad m = \mbox{rank}(XY^T), \quad t = \mbox{min}\{r,s\}.
$$
Let the (reduced) SVD factorizations of $X$ and $Y$ be, respectively,
\begin{equation}\label{XQR}
    X = U\begin{bmatrix}
           \Sigma_{1} \\
           0 \\
         \end{bmatrix} Q_1^T
      =
        \begin{bmatrix}
          U_{1} & U_{2} \\
        \end{bmatrix}
        \begin{bmatrix}
          \Sigma_{1} \\
          0 \\
        \end{bmatrix} Q_1^T
      = U_{1}\Sigma_{1} Q_1^T,
\end{equation}
and
\begin{equation}\label{YQR}
    Y = V\begin{bmatrix}
           \Sigma_{2} \\
           0 \\
         \end{bmatrix} Q_2^T
      =
        \begin{bmatrix}
          V_{1} & V_{2} \\
        \end{bmatrix}
        \begin{bmatrix}
          \Sigma_{2} \\
          0 \\
        \end{bmatrix} Q_2^T
      = V_{1}\Sigma_{2} Q_2^T,
\end{equation}
where
$$
U \in \mathbf{R}^{d_{1}\times d_{1}}, \ U_{1} \in \mathbf{R}^{d_{1}\times r}, \ U_{2} \in \mathbf{R}^{d_{1}\times (d_{1}-r)}, \
\Sigma_{1} \in \mathbf{R}^{r \times r}, \ Q_{1} \in \mathbf{R}^{n \times r},
$$
$$
V \in \mathbf{R}^{d_{2}\times d_{2}}, \ V_{1} \in \mathbf{R}^{d_{2}\times s}, \ V_{2} \in \mathbf{R}^{d_{2}\times (d_{2}-s)}, \
\Sigma_{2} \in \mathbf{R}^{s \times s}, \ Q_{2} \in \mathbf{R}^{n \times s},
$$
$U$ and $V$ are orthogonal, $\Sigma_1$ and $\Sigma_2$ are nonsingular and diagonal, $Q_{1}$ and $Q_{2}$ are column orthogonal. It
follows from the two orthogonality constraints in \eqref{CompactCCA} that
\begin{equation}\label{lrange}
    l \leq \mbox{min}\{\mbox{rank}(X),\mbox{rank}(Y)\} = \mbox{min}\{r,s\}=t.
\end{equation}
Next, let
\begin{equation}\label{SVDQ1Q2}
  Q^{T}_{1}Q_{2} = P_{1}\Sigma P^{T}_{2}
\end{equation}
be the singular value decomposition of $Q^{T}_{1}Q_{2}$, where $P_{1} \in \mathbf{R}^{r \times r}$ and $P_{2} \in \mathbf{R}^{s
\times s}$ are orthogonal, $\Sigma\in \mathbf{R}^{r\times s}$, and assume there are $q$ distinctive nonzero singular values with
multiplicity $m_{1},m_{2},\cdots,m_{q}$, respectively, then
$$m = \sum\limits^{q}_{i=1}m_{i}=\mbox{rank}(Q_1^TQ_2)\leq \mbox{min}\{r,s\}=t. $$
The full characterization of $W_{x}$ and $W_{y}$ is given in the following theorem \cite{ChuLNZ13}.

\begin{theorem}\label{GeneralSolutionThm}
i). If $l = \sum\limits^{k}_{i=1}m_{i}$ for some $k$ satisfying $1 \leq k \leq q$, then $(W_{x},W_{y})$ with $W_x\in
\mathbf{R}^{d_1\times l}$ and $W_y\in \mathbf{R}^{d_2\times l}$ is a solution of optimization problem \eqref{CompactCCA} if and
only if
\begin{equation}\label{WxWyCase1}
    \left\{
       \begin{array}{l}
         W_{x} = U_{1}\Sigma^{-1}_{1}P_{1}(:,1:l)\mathcal{W} + U_{2}\mathcal{E},\\
         W_{y} = V_{1}\Sigma^{-1}_{2}P_{2}(:,1:l)\mathcal{W} + V_{2}\mathcal{F},
       \end{array}
    \right.
\end{equation}
where $\mathcal{W} \in \mathbf{R}^{l\times l}$ is orthogonal, $\mathcal{E} \in \mathbf{R}^{(d_{1}-r)\times l}$ and $\mathcal{F}
\in \mathbf{R}^{(d_{2}-s)\times l}$ are arbitrary.

ii). If $\sum\limits^{k}_{i=1}m_{i} < l < \sum\limits^{k+1}_{i=1}m_{i}$ for some $k$ satisfying $0 \leq k < q$, then
$(W_{x},W_{y})$ with $W_x\in \mathbf{R}^{d_1\times l}$ and $W_y\in \mathbf{R}^{d_2\times l}$ is a solution of optimization
problem \eqref{CompactCCA} if and only if
\begin{equation}\label{WxWyCase2}
    \left\{
       \begin{array}{l}
         W_{x} = U_{1}\Sigma^{-1}_{1}\begin{bmatrix}
                                            P_{1}(:,1:\alpha_{k})
                                            & P_{1}(:,1 + \alpha_{k}:\alpha_{k+1})\mathcal{G} \\
                                          \end{bmatrix}\mathcal{W} + U_{2}\mathcal{E}, \\
         W_{y} = V_{1}\Sigma^{-1}_{2}\begin{bmatrix}
                                            P_{2}(:,1:\alpha_{k})
                                            & P_{2}(:,1 + \alpha_{k}:\alpha_{k+1})\mathcal{G} \\
                                          \end{bmatrix}\mathcal{W} + V_{2}\mathcal{F},
       \end{array}
    \right.
\end{equation}
where $\alpha_{k} = \sum\limits^{k}_{i=1}m_{i} ~for~ k=1,\cdots,q$, $\mathcal{W} \in \mathbf{R}^{l\times l}$ is orthogonal,
$\mathcal{G} \in \mathbf{R}^{m_{(k+1)}\times (l - \alpha_{k})}$ is column orthogonal, $\mathcal{E} \in
\mathbf{R}^{(d_{1}-r)\times l}$ and $\mathcal{F} \in \mathbf{R}^{(d_{2}-s)\times l}$ are arbitrary.

iii). If $m < l \leq \min\{r,s\}$, then $(W_{x},W_{y})$ with $W_x\in \mathbf{R}^{d_1\times l}$ and $W_y\in \mathbf{R}^{d_2\times
l}$ is a solution of optimization problem \eqref{CompactCCA} if and only if
\begin{equation}\label{WxWyCase3}
    \left\{
       \begin{array}{l}
         W_{x} = U_{1}\Sigma^{-1}_{1}\begin{bmatrix}
                                            P_{1}(:,1:m) & P_{1}(:,m+1:r)\mathcal{G}_1 \\
                                          \end{bmatrix}\mathcal{W} + U_{2}\mathcal{E}, \\
         W_{y} = V_{1}\Sigma^{-1}_{2}\begin{bmatrix}
                                            P_{2}(:,1:m) & P_{2}(:,m+1:s)\mathcal{G}_2 \\
                                          \end{bmatrix}\mathcal{W} + V_{2}\mathcal{F},
       \end{array}
    \right.
\end{equation}
where $\mathcal{W} \in \mathbf{R}^{l\times l}$ is orthogonal, $\mathcal{G}_{1} \in \mathbf{R}^{(r-m)\times (l-m)}$ and
$\mathcal{G}_{2} \in \mathbf{R}^{(s-m)\times (l-m)}$ are column orthogonal, $\mathcal{E} \in \mathbf{R}^{(d_{1}-r)\times l}$ and
$\mathcal{F} \in \mathbf{R}^{(d_{2}-s)\times l}$ are arbitrary.
\end{theorem}

An immediate application of Theorem \ref{GeneralSolutionThm} is that we can prove that Uncorrelated Linear Discriminant Analysis
(ULDA) \cite{ChuTH11,JinYHL01,Ye05} is a special case of CCA when one set of variables is derived form the data matrix and the
other set of variables is constructed from class information. This theorem has also been utilized in \cite{ChuLNZ13} to design a
sparse CCA algorithm.

\subsection{Kernel canonical correlation analysis}

Now, we look at some details on the derivation of kernel CCA. Note from Theorem \ref{GeneralSolutionThm} that each solution
$(W_{x},W_{y})$ of CCA can be expressed as
$$
W_{x} = X\mathcal{W}_{x} + W^{\perp}_{x}, \quad W_{y} = Y\mathcal{W}_{y} + W^{\perp}_{y},
$$
where $W^{\perp}_{x}$ and $W^{\perp}_{y}$ are orthogonal to the range space of $X$ and $Y$, respectively. Since, intrinsically,
kernel CCA is performing ordinary CCA on $\Phi_{x}$ and $\Phi_{y}$, it follows that the solutions of kernel CCA should be
obtained by virtually solving
\begin{equation}\label{vKCCA}
\begin{array}{rl}
   \max\limits_{W_{x},W_{y}} & \mbox{Trace}(W^{T}_{x}\Phi_{x}\Phi_{y}W_{y}) \\
   s.t. & W^{T}_{x}\Phi_{x}\Phi^{T}_{x}W_{x} = I,~W_{x}\in \mathbf{R}^{\mathcal{N}_{x}\times l}, \\
        & W^{T}_{y}\Phi_{y}\Phi^{T}_{y}W_{y} = I,~W_{y}\in \mathbf{R}^{\mathcal{N}_{y}\times l},
\end{array}
\end{equation}
Similar to ordinary CCA, each solution $(W_{x},W_{y})$ of \eqref{vKCCA} shall be represented as
\begin{equation}\label{KCCASolution}
   W_{x} = \Phi_{x}\mathcal{W}_{x} + W^{\perp}_{x}, \quad W_{y} = \Phi_{y}\mathcal{W}_{y} + W^{\perp}_{y},
\end{equation}
where $\mathcal{W}_{x},~\mathcal{W}_{y}\in \mathbf{R}^{n\times l}$ are usually called \textit{dual transformation matrices},
$W^{\perp}_{x}$ and $W^{\perp}_{y}$ are orthogonal to the range space of $\Phi_{x}$ and $\Phi_{y}$, respectively.

Substituting \eqref{KCCASolution} into \eqref{vKCCA}, we have
$$
W^{T}_{x}\Phi_{x}\Phi_{y}W_{y} = \mathcal{W}^{T}_{x}K_{x}K_{y}\mathcal{W}_{y}, \quad %
W^{T}_{x}\Phi_{x}\Phi^{T}_{x}W_{x} = \mathcal{W}^{T}_{x}K^{2}_{x}\mathcal{W}_{x}, \quad %
W^{T}_{y}\Phi_{y}\Phi^{T}_{y}W_{y} = \mathcal{W}^{T}_{y}K^{2}_{y}\mathcal{W}_{y}.
$$
Thus, the computation of transformations of kernel CCA can be converted to the computation of dual transformation matrices
$\mathcal{W}_{x}$ and $\mathcal{W}_{y}$ by solving the following optimization problem
\begin{equation}\label{KCCA}
\begin{array}{rl}
   \max\limits_{\mathcal{W}_{x},\mathcal{W}_{y}} & \mbox{Trace}(\mathcal{W}^{T}_{x}K_{x}K_{y}\mathcal{W}_{y}) \\
   s.t. & \mathcal{W}^{T}_{x}K^{2}_{x}\mathcal{W}_{x} = I,~\mathcal{W}_{x}\in \mathbf{R}^{n\times l}, \\
        & \mathcal{W}^{T}_{y}K^{2}_{y}\mathcal{W}_{y} = I,~\mathcal{W}_{y}\in \mathbf{R}^{n\times l},
\end{array}
\end{equation}
which is used as the criterion of kernel CCA in this paper.

As can be seen from the analysis above, terms $W^{\perp}_{x}$ and $W^{\perp}_{y}$ in \eqref{KCCASolution} do not contribute to
the canonical correlations between $\Phi_{x}$ and $\Phi_{y}$, thus, are usually neglected in practice. Therefore, when we are
given a set of testing data
$X_{t}=\begin{bmatrix} x^{1}_{t} & \cdots & x^{N}_{t} \\
\end{bmatrix}$
consisting of $N$ points, the projection of $X_{t}$ onto kernel CCA direction $W_{x}$ can be performed by first mapping $X_{t}$
into feature space $\mathcal{H}_{x}$, then compute its inner product with $W_{x}$. More specifically, suppose $
\Phi_{x,t}=\begin{bmatrix}
             \phi_{x}(x^{1}_{t}) & \cdots & \phi_{x}(x^{N}_{t}) \\
           \end{bmatrix}$
is the projection of $X_{t}$ in feature space $\mathcal{H}_{x}$, then the projection of $X_{t}$ onto kernel CCA direction $W_{x}$
is given by
$$
W^{T}_{x}\Phi_{x,t} =\mathcal{W}^{T}_{x}K_{x,t},
$$
where $K_{x,t}=\langle\Phi_{x},\Phi_{x,t}\rangle=[\kappa_{x}(x_{i},x^{j}_{t})]^{j=1:N}_{i=1:n} \in \mathbf{R}^{n\times N}$ is the
matrix consisting of the kernel evaluations of $X_{t}$ with all training data $X$. Similar process can be adopted to compute
projections of new data drawn from variable $y$.

In the process of deriving \eqref{KCCA}, we assumed data $\Phi_{x}$ and $\Phi_{y}$ have been centered (that is, the column mean
of both $\Phi_{x}$ and $\Phi_{y}$ are zero), otherwise, we need to perform data centering before applying kernel CCA. Unlike data
centering of $X$ and $Y$, we can not perform data centering directly on $\Phi_{x}$ and $\Phi_{y}$ since we do not know their
explicit coordinates. However, as shown in \cite{Scholkopf98,Scholkopf02}, data centering in \textit{RKHS} can be accomplished
via some operations on kernel matrices. To center $\Phi_{x}$, a natural idea should be computing $\Phi_{x,c} = \Phi_{x}(I -
\frac{e_{n}e_{n}^{T}}{n})$, where $e_{n}$ denotes column vector in $\mathbf{R}^{n}$ with all entries being 1. However, since
kernel CCA makes use of the data through kernel matrix $K_{x}$, the centering process can be performed on $K_{x}$ as
\begin{equation}\label{CKx}
K_{x,c} = \langle\Phi_{x,c},\Phi_{x,c}\rangle = (I - \frac{e_{n}e_{n}^{T}}{n})\langle\Phi_{x},\Phi_{x}\rangle(I -
\frac{e_{n}e_{n}^{T}}{n}) = (I - \frac{e_{n}e_{n}^{T}}{n})K_{x}(I - \frac{e_{n}e_{n}^{T}}{n}).
\end{equation}
Similarly, we can center testing data as
\begin{equation}\label{CKxt}
K_{x,t,c} = \langle\Phi_{x,c}, \Phi_{x,t}-\Phi_{x}\frac{e_{n}e^{T}_{N}}{n}\rangle %
=(I - \frac{e_{n}e_{n}^{T}}{n})K_{x,t} - (I - \frac{e_{n}e_{n}^{T}}{n})K_{x}\frac{e_{n}e^{T}_{N}}{n}.
\end{equation}
More details about data centering in \textit{RKHS} can be found in \cite{Scholkopf98,Scholkopf02}. In the sequel of this paper,
we assume the given data have been centered.

There are papers studying properties of kernel CCA, including the geometry of kernel CCA in \cite{Kuss03} and statistical
consistency of kernel CCA in \cite{Fukumizu07}. In the remainder of this paper, we consider sparse kernel CCA. Before that, we
explore a relation between CCA and least squares in the next section.

\section{Sparse CCA based on least squares formulation}\label{LSCCA}
\setcounter{equation}{0}

Note form \eqref{singleCCAequi} that when one of $X$ and $Y$ is one dimensional, CCA is equivalent to least squares estimation to
a linear regression problem. For more general cases, some relation between CCA and linear regression has been established under
the condition that $\text{rank}(X)=n-1$ and $\text{rank}(Y) = d_{2}$ in \cite{SunJY11}. In this section, we establish a relation
between CCA and linear regression without any additional constraint on $X$ and $Y$. Moreover, based on this relation we design a
new sparse CCA algorithm.

We focus on a solution subset of optimization problem \eqref{CompactCCA} presented in the following lemma, whose proof is trivial
and omitted.
\begin{lemma}\label{lemma}
Any $(W_{x},W_{y})$ of the following forms
\begin{equation}\label{lemma1}
    \left\{
       \begin{array}{l}
         W_{x} = U_{1}\Sigma^{-1}_{1}P_{1}(:,1:l) + U_{2}\mathcal{E},\\
         W_{y} = V_{1}\Sigma^{-1}_{2}P_{2}(:,1:l) + V_{2}\mathcal{F},
       \end{array}
    \right.
\end{equation}
is a solution of optimization problem \eqref{CompactCCA}, where $\mathcal{E} \in \mathbf{R}^{(d_{1}-r)\times l}$ and $\mathcal{F}
\in \mathbf{R}^{(d_{2}-s)\times l}$ are arbitrary.
\end{lemma}

Suppose matrix factorizations \eqref{XQR}-\eqref{SVDQ1Q2} have been accomplished, and let
\begin{eqnarray}
T_{x}=Y^{T}[(YY^{T})^{\frac{1}{2}}]^{\dagger}V_{1}P_2(:,1:l)\Sigma(1:l,1:l)^{-1} = Q_{2}P_2(:,1:l)\Sigma(1:l,1:l)^{-1}, \label{PridictorX} \\
T_{y}=X^{T}[(XX^{T})^{\frac{1}{2}}]^{\dagger}U_{1}P_1(:,1:l)\Sigma(1:l,1:l)^{-1} = Q_{1}P_1(:,1:l)\Sigma(1:l,1:l)^{-1},
\label{PridictorY}
\end{eqnarray}
where $A^{\dagger}$ denotes the \textit{Moore-Penrose} inverse of a general matrix $A$ and $1 \leq l \leq m$, then we have the
following theorem.

\begin{theorem}\label{CCALSformula}
For any $l$ satisfying $1 \leq l \leq m$, suppose $W_{x} \in \mathbf{R}^{d_{1}\times l}$ and $W_{y} \in \mathbf{R}^{d_{2}\times
l}$ satisfy
\begin{equation}\label{LSWx}
    W_{x} = \text{arg}\min\{\|X^{T}W_{x}-T_{x}\|_{F}^{2}: W_{x} \in \mathbf{R}^{d_{1}\times l}\},
\end{equation}
and
\begin{equation}\label{LSWy}
    W_{y} = \text{arg}\min\{\|Y^{T}W_{x}-T_{y}\|_{F}^{2}: W_{y} \in \mathbf{R}^{d_{2}\times l}\},
\end{equation}
where $T_{x}$ and $T_{y}$ are defined in \eqref{PridictorX} and \eqref{PridictorY}, respectively. Then $W_{x}$ and $W_{y}$ form a
solution of optimization problem \eqref{CompactCCA}.
\end{theorem}

\begin{proof}
Since \eqref{LSWx} and \eqref{LSWy} have the same form, we only prove the result for $W_{x}$, the same idea can be applied to
$W_{y}$.

We know that $W_{x}$ is a solution of \eqref{LSWx} if and only if it satisfies the normal equation
\begin{equation}\label{WxNormalEq}
XX^{T}W_{x}=XT_{x}.
\end{equation}
Substituting factorizations \eqref{XQR}, \eqref{YQR} and \eqref{SVDQ1Q2} into the equation above, we get
$$
XX^{T}=U_{1}\Sigma_{1}^{2}U_{1}^{T},
$$
and
\begin{eqnarray*}
  XT_{x} &=& U_{1}\Sigma_{1}Q_{1}^{T}Q_{2}P_{2}(:,1:l)\Sigma(1:l,1:l)^{-1} \\
         &=& U_{1}\Sigma_{1}P_{1}(:,1:l),
\end{eqnarray*}
which yield an equivalent reformulation of \eqref{WxNormalEq}
\begin{equation}\label{WxNormalEqequi}
U_{1}\Sigma_{1}^{2}U_{1}^{T}W_{x}=U_{1}\Sigma_{1}P_{1}(:,1:l).
\end{equation}
It is easy to check that $W_{x}$ is a solution of \eqref{WxNormalEqequi} if and only if
\begin{equation}\label{LSWxSolution}
    W_{x}=U_{1}\Sigma_{1}^{-1}P_{1}(:,1:l) + U_{2}\mathcal{E},
\end{equation}
where $\mathcal{E} \in \mathbf{R}^{(d_{1}-r)\times l}$ is an arbitrary matrix. Therefore, $W_{x}$ is a solution of \eqref{LSWx}
if and only if $W_{x}$ can be formulated as \eqref{LSWxSolution}.

Similarly, $W_{y}$ is a solution of \eqref{LSWy} if and only if $W_{y}$ can be written as
\begin{equation}\label{LSWySolution}
    W_{y}=V_{1}\Sigma_{2}^{-1}P_{2}(:,1:l) + V_{2}\mathcal{F},
\end{equation}
where $\mathcal{F} \in \mathbf{R}^{(d_{2}-s)\times l}$ is an arbitrary matrix.

Now, comparing equations \eqref{LSWxSolution} and \eqref{LSWySolution} with the equation \eqref{lemma1} in Lemma \ref{lemma}, we
can conclude that for any solution $W_{x}$ of the least squares problem \eqref{LSWx} and any solution $W_{y}$ of the least
squares problem \eqref{LSWy}, $W_{x}$ and $W_{y}$ form a solution of optimization problem \eqref{CompactCCA}, hence a solution of
CCA.
\end{proof}

\begin{remark}
In Theorem \ref{CCALSformula} we only consider $l$ satisfying $1 \leq l \leq m$. This is reasonable, since there are $m$ nonzero
canonical correlations between $X$ and $Y$, and weight vectors corresponding to zero canonical correlation does not contribute to
the correlation between data $X$ and $Y$.
\end{remark}

Consider the usual regression situation: we have a set of observations $(x_{1},b_{1})\cdots(x_{n},b_{n})$ where $x_{i} \in
\mathbf{R}^{d_{1}}$ and $b_{i}$ are the regressor and response for the $i$th observation. Suppose $\{x_{i}\}$ has been centered,
then linear regression model has the form
$$
f(X)=\sum\limits^{n}_{i=1}x_{i}\beta_{i},
$$
and aims to estimate $\beta = \begin{bmatrix}
                                \beta_{1} & \cdots & \beta_{n} \\
                              \end{bmatrix}$
so as to predict an output for each input $x$. The famous least squares estimation minimizes the residual sum of squares
$$
Res(\beta)=\|X^{T}\beta - b\|^{2}_{2}.
$$
Therefore, \eqref{LSWx} and \eqref{LSWy} can be interpreted as least squares estimations of linear regression problems with
columns of $X$ and $Y$ being regressors and rows of $T_{x}$ and $T_{y}$ being corresponding responses.

Recent research on lasso \cite{Tibshirani96} shows that simultaneous sparsity and regression can be achieved by penalizing the
$\ell_{1}$-norm of the variables. Motivated by this, we incorporate sparsity into CCA via the established relationship between
CCA and least squares and considering the following $\ell_{1}$-norm penalized least squares problems
\begin{equation}\label{SparseCCAWx}
    \min_{W_{x}}\{\frac{1}{2}\|X^{T}W_{x}-T_{x}\|_{F}^{2} + \sum\limits^{l}_{i=1}\lambda_{x,i}\|W_{x,i}\|_{1}: W_{x} \in \mathbf{R}^{d_{1}\times l}\},
\end{equation}
and
\begin{equation}\label{SparseCCAWy}
    \min_{W_{y}}\{\frac{1}{2}\|Y^{T}W_{y}-T_{y}\|_{F}^{2} + \sum\limits^{l}_{i=1}\lambda_{y,i}\|W_{y,i}\|_{1}: W_{y} \in \mathbf{R}^{d_{2}\times l}\},
\end{equation}
where $\lambda_{x,i}$, $\lambda_{y,i}$ are positive regularization parameters and $W_{x,i}$, $W_{y,i}$ are the $i$th column of
$W_{x}$ and $W_{y}$, respectively. When we set $\lambda_{x,1}=\cdots=\lambda_{x,l}=\lambda_{x} > 0$ and
$\lambda_{y,1}=\cdots=\lambda_{y,l}=\lambda_{y} > 0$, problems \eqref{SparseCCAWx} and \eqref{SparseCCAWy} become
\begin{equation}\label{SparseCCAWx2}
    \min_{W_{x}}\{\frac{1}{2}\|X^{T}W_{x}-T_{x}\|_{F}^{2} + \lambda_{x}\|W_{x}\|_{1}: W_{x} \in \mathbf{R}^{d_{1}\times l}\},
\end{equation}
and
\begin{equation}\label{SparseCCAWy2}
    \min_{W_{y}}\{\frac{1}{2}\|Y^{T}W_{y}-T_{y}\|_{F}^{2} + \lambda_{y}\|W_{y}\|_{1}: W_{y} \in \mathbf{R}^{d_{2}\times l}\},
\end{equation}
where
\begin{eqnarray*}
  \|W_{x}\|_{1} = \sum\limits^{d_{1}}_{i=1}\sum\limits^{l}_{j=1}|W_{x}(i,j)|, & \|W_{y}\|_{1} =
  \sum\limits^{d_{2}}_{i=1}\sum\limits^{l}_{j=1}|W_{y}(i,j)|.
\end{eqnarray*}

Since \eqref{SparseCCAWx} and \eqref{SparseCCAWy} (also, \eqref{SparseCCAWx2} and \eqref{SparseCCAWy2})have the same form, all
results holding for one problem can be naturally extended to the other, so we concentrate on \eqref{SparseCCAWx}. Optimization
problem \eqref{SparseCCAWx} reduces to a $\ell_{1}$-regularized minimization problem of the form
\begin{equation}\label{L1problem}
    \min_{x \in \mathbf{R}^{d}}\frac{1}{2}\|Ax-b\|^{2}_{2} + \lambda \|x\|_{1},
\end{equation}
when $l=1$. In the field of compressed sensing, \eqref{L1problem} has been intensively studied as \textit{denoising basis
pursuit} problem, and many efficient approaches have been proposed to solve it, see \cite{Beck09,Figueiredo07,Hale08,YangZ11}. In
this paper we adopt the fixed-point continuation (FPC) method \cite{Hale08,Hale10}, due to its simple implementation and nice
convergence property.

Fixed-point algorithm for \eqref{L1problem} is an iterative method which updates iterates as
\begin{equation}\label{FPitr}
x^{k+1} = \mathcal{S}_{\nu}\left(x^{k} - \tau A^{T}(Ax-b)\right), \text{~with~} \nu = \tau \lambda,
\end{equation}
where $\tau > 0$ denotes the step size, and $\mathcal{S}_{\nu}$ is the soft-thresholding operator defined as
$$
\mathcal{S}_{\nu}(x) = \begin{bmatrix}
                         \mathcal{S}_{\nu}(x_{1}) & \cdots & \mathcal{S}_{\nu}(x_{d}) \\
                       \end{bmatrix}^{T}
$$
with
\begin{equation}\label{shrinkage}
\mathcal{S}_{\nu}(\omega) = \text{sign}(\omega) \max\{|\omega|-\nu, 0\}, ~\omega \in \mathbf{R}.
\end{equation}
$\mathcal{S}_{\nu}(\omega)$ reduces any $\omega$ with magnitude less than $\nu$ to zero, thus reducing the $\ell_{1}$-norm and
introducing sparsity.

The fixed-point algorithm can be naturally extended to solve \eqref{SparseCCAWx}, which yields
\begin{equation}\label{FPWx}
W^{k+1}_{x,i} = \mathcal{S}_{\nu_{x,i}}\left(W^{k}_{x,i} - \tau_{x} X(X^{T}W^{k}_{x,i} - T_{x,i})\right), ~i=1,\cdots,l,
\end{equation}
where $\nu_{x,i} = \tau_{x} \lambda_{x,i}$ with $\tau_{x} > 0$ denoting the step size. We can prove that fixed-point iterations
have some nice convergence properties which are presented in the following theorem.

\begin{theorem}\label{FPConvergence}\cite{Hale08}
Let $\Omega$ be the solution set of \eqref{SparseCCAWx}, then there exists $M^{*} \in \mathbf{R}^{d_{1}\times l}$ such that
\begin{equation}\label{Optimalgradient}
    X(X^{T}W_{x} - T_{x}) \equiv M^{*}, ~\forall W_{x} \in \Omega.
\end{equation}
In addition, define
\begin{equation}\label{IndexSet}
    L:= \{(i,j): |M^{*}_{i,j}| < \lambda_{x}\}                       %% $1\leq i \leq d_{1}, 1\leq j \leq l$
\end{equation}
as a subset of indices and let $\lambda_{max}(XX^{T})$ be the maximum eigenvalue of $XX^{T}$, and choose $\tau_{x}$ from
$$
0 < \tau_{x} < \frac{2}{\lambda_{max}(XX^{T})},
$$
then the sequence $\{W^{k}_{x}\}$, generated by the fixed-point iterations \eqref{FPWx} starting with any initial point
$W^{0}_{x}$, converges to some $W^{\ast}_{x} \in \Omega$. Moreover, there exists an integer $K>0$ such that
\begin{equation}\label{FiniteConverge}
    (W^{k}_{x})_{i,j} = (W^{\ast}_{x})_{i,j} = 0, ~\forall (i,j) \in L,
\end{equation}
when $k>K$.
\end{theorem}

\begin{remark}
\begin{enumerate}
  \item Equation \eqref{Optimalgradient} shows that for any two optimal solutions of \eqref{SparseCCAWx} the gradient of the squared
Frobenius norm in \eqref{SparseCCAWx} must be equal.
  \item Equation \eqref{FiniteConverge} means that the entries of $W^{k}_{x}$ with indices from $L$ will converge to zero in
finite steps. The positive integer $K$ is a function of $W^{0}_{x}$ and $W^{\ast}_{x}$, and determined by the distance between
them.
\end{enumerate}
\end{remark}

Similarly, we can design a fixed-point algorithm to solve \eqref{SparseCCAWy} as follows:
\begin{equation}\label{FPWy}
W^{k+1}_{y,i} = \mathcal{S}_{\nu_{y,i}}\left(W^{k}_{y,i} - \tau_{y} Y(Y^{T}W^{k}_{y,i} - T_{y,i})\right), \text{~with~} \nu_{y,i}
= \tau_{y} \lambda_{y,i},~i=1,\cdots,l,
\end{equation}
where $\tau_{y} > 0$ denotes the step size.

Now, we are ready to present our sparse CCA algorithm.
\begin{algorithm}[H]
\caption{(SCCA$\_$LS: Sparse CCA based on least squares)} \label{SCCAFPC}
\begin{algorithmic}[1]
\REQUIRE Training data $X \in \mathbf{R}^{d_{1}\times n}$, $Y \in \mathbf{R}^{d_{2}\times n}$                     % Input

\ENSURE Sparse transformation matrices $W_{x}\in \mathbf{R}^{d_1\times l}$ and $W_{y}\in \mathbf{R}^{d_2\times l}$.
% Output

\STATE Compute matrix factorizations \eqref{XQR}-\eqref{SVDQ1Q2};

\STATE Compute $T_{x}$ and $T_{y}$ according to \eqref{PridictorX} and \eqref{PridictorY};

\REPEAT
    \STATE $W^{k+1}_{x,i} = \mathcal{S}_{\nu_{x,i}}\left(W^{k}_{x,i} - \tau_{x} X(X^{T}W^{k}_{x,i} - T_{x,i})\right)$,
           $\nu_{x,i}=\tau_{x} \lambda_{x,i}$, $i=1,\cdots,l$,
\UNTIL {convergence}

\REPEAT
    \STATE $W^{k+1}_{y,i} = \mathcal{S}_{\nu_{y,i}}\left(W^{k}_{y,i} - \tau_{y} Y(Y^{T}W^{k}_{y,i} - T_{y,i})\right)$,
           $\nu_{x,i}=\tau_{x} \lambda_{x,i}$, $i=1,\cdots,l$,
\UNTIL {convergence}

\RETURN $W_{x} = W^{k}_{x}$ and $W_{y} = W^{k}_{y}$.
\end{algorithmic}
\end{algorithm}

Although different solutions may be returned by Algorithm \ref{SCCAFPC} starting from different initial points, we can conclude
form \eqref{Optimalgradient} that
$$
XX^{T}W^{\ast}_{x} = XX^{T}\widehat{W}^{\ast}_{x} , ~\forall W^{\ast}_{x}, \widehat{W}^{\ast}_{x} \in \Omega,
$$
which results in $U^{T}_{1}W^{\ast}_{x} = U^{T}_{1}\widehat{W}^{\ast}_{x}$. Similarly, we have $V^{T}_{1}W^{\ast}_{y} =
V^{T}_{1}\widehat{W}^{\ast}_{y}$ for two different solutions of \eqref{SparseCCAWy}. Hence,
\begin{align*}
  (W^{\ast}_{x})^{T}XX^{T}W^{\ast}_{x} &= (\widehat{W}^{\ast}_{x})^{T}XX^{T}\widehat{W}^{\ast}_{x}, \\
  (W^{\ast}_{y})^{T}YY^{T}W^{\ast}_{y} &= (\widehat{W}^{\ast}_{y})^{T}YY^{T}\widehat{W}^{\ast}_{y}, \\
  (W^{\ast}_{x})^{T}XY^{T}W^{\ast}_{y} &= (\widehat{W}^{\ast}_{x})^{T}XY^{T}\widehat{W}^{\ast}_{y}.
\end{align*}
The above equations show that any two optimal solutions of \eqref{SparseCCAWx} approximate the solution of CCA in the same level.

Due to the effect of $\ell_{1}$-norm regularization, a solution $(W^{\ast}_{x}, W^{\ast}_{y})$ does not satisfy the orthogonality
constraints of CCA any more, but we can derive a bound on the deviation. Since \eqref{SparseCCAWx} is a convex optimization
problem, we have
\begin{equation}\label{OptimalCon}
    X(X^{T}W^{\ast}_{x} - T_{x}) + \lambda_{x}\mathcal{G} = 0, \text{~for~some~} \mathcal{G} \in \partial\|W^{\ast}_{x}\|_{1},
\end{equation}
where $\partial\|W^{\ast}_{x}\|_{1}$ denotes the sub-differential of $\|\cdot\|_{1}$ at $W^{\ast}_{x}$. Simplifying
\eqref{OptimalCon}, we can get
$$
U^{T}_{1}W^{\ast}_{x} = \Sigma^{-1}_{1}P_{1}(:,1:l) - \lambda_{x}\Sigma^{-2}_{1}U^{T}_{1}\mathcal{G},
$$
which implies
$$
(W^{\ast}_{x})^{T}XX^{T}W^{\ast}_{x} = I_{l} - \lambda_{x}P_{1}(:,1:l)^{T}\Sigma^{-1}_{1}U^{T}_{1}\mathcal{G} -
\lambda_{x}\mathcal{G}^{T}U_{1}\Sigma^{-1}_{1}P_{1}(:,1:l) +
\lambda^{2}_{x}\mathcal{G}^{T}U_{1}\Sigma^{-2}_{1}U^{T}_{1}\mathcal{G}.
$$
Since $\mathcal{G}\in \mathbf{R}^{d_{1}\times l}$ satisfies $|\mathcal{G}_{i,j}| \leq 1$ for $i=1,\cdots,d_1,~j=1,\cdots,l$, we
further assume there are $N_{x}$ non-zeros in $\mathcal{G}$, it follows that
\begin{align}\label{Orthboundx}
    \frac{\|(W^{\ast}_{x})^{T}XX^{T}W^{\ast}_{x} - I_{l}\|_{F}}{\sqrt{l}}
    & \leq \frac{\lambda_{x}}{\sigma_{r}(X)\sqrt{l}}\left( 2\sqrt{N_{x}}+\frac{\lambda_{x}}{\sigma_{r}(X)}N_{x} \right) \nonumber \\
    & \leq \frac{\lambda_{x}\sqrt{d_{1}}}{\sigma_{r}(X)} \left( 2+\frac{\lambda_{x}}{\sigma_{r}(X)}\sqrt{ld_{1}} \right),
\end{align}
where $\sigma_{r}(X)$ denotes the smallest nonzero singular value of $X$. So the bound is affected by regularization parameter
$\lambda_{x}$, the smallest nonzero singular value of $X$ and the number of non-zeros in $\mathcal{G}$. A Similar result can be
obtained for the optimal solutions of \eqref{SparseCCAWy}.

\section{Extension to kernel canonical correlation analysis}\label{KernelCCA}
\setcounter{equation}{0}

Since kernel CCA criterion \eqref{KCCA} and CCA criterion \eqref{CompactCCA} have the same form, we can expect a similar
characterization of solutions of \eqref{KCCA} as Theorem \ref{GeneralSolutionThm}. Define
$$
\hat{r} = \mbox{rank}(K_{x}), \quad \hat{s} = \mbox{rank}(K_{y}), \quad \hat{m} = \mbox{rank}(K_{x}K^T_{y}),
$$
and let the eigenvalue decomposition of $K_{x}$ and $K_{y}$ be, respectively,
\begin{equation}\label{KxEigenDec}
    K_{x} = \mathcal{U}\begin{bmatrix}
           \Pi_{1} & 0 \\
           0 & 0 \\
         \end{bmatrix} \mathcal{U}^T
      =
        \begin{bmatrix}
          \mathcal{U}_{1} & \mathcal{U}_{2} \\
        \end{bmatrix}
        \begin{bmatrix}
          \Pi_{1} & 0 \\
          0 & 0 \\
        \end{bmatrix}
        \begin{bmatrix}
          \mathcal{U}_{1} & \mathcal{U}_{2} \\
        \end{bmatrix}^{T}
      = \mathcal{U}_{1}\Pi_{1} \mathcal{U}^T_{1},
\end{equation}
and
\begin{equation}\label{KyEigenDec}
    K_{y} = \mathcal{V}\begin{bmatrix}
           \Pi_{2} & 0 \\
           0 & 0 \\
         \end{bmatrix} \mathcal{V}^T
      =
        \begin{bmatrix}
          \mathcal{V}_{1} & \mathcal{V}_{2} \\
        \end{bmatrix}
        \begin{bmatrix}
          \Pi_{2} & 0 \\
          0 & 0 \\
        \end{bmatrix}
        \begin{bmatrix}
          \mathcal{V}_{1} & \mathcal{V}_{2} \\
        \end{bmatrix}^{T}
      = \mathcal{V}_{1}\Pi_{2} \mathcal{V}^T_{1},
\end{equation}
where
$$
\mathcal{U} \in \mathbf{R}^{n\times n}, \ \mathcal{U}_{1} \in \mathbf{R}^{n\times \hat{r}}, \ \mathcal{U}_{2} \in
\mathbf{R}^{n\times (n-\hat{r})}, \ \Pi_{1} \in \mathbf{R}^{\hat{r} \times \hat{r}},
$$
$$
\mathcal{V} \in \mathbf{R}^{n\times n}, \ \mathcal{V}_{1} \in \mathbf{R}^{n\times \hat{s}}, \ \mathcal{V}_{2} \in
\mathbf{R}^{n\times (n-\hat{s})}, \ \Pi_{2} \in \mathbf{R}^{\hat{s} \times \hat{s}},
$$
$\mathcal{U}$ and $\mathcal{V}$ are orthogonal, $\Pi_{1}$ and $\Pi_{2}$ are nonsingular and diagonal. In addition, let
\begin{equation}\label{KxKySVD}
\mathcal{U}^T_{1}\mathcal{V}_{1} = \mathcal{P}_{1} \Pi \mathcal{P}^{T}_{2}
\end{equation}
be the singular value decomposition of $\mathcal{U}^T_{1}\mathcal{V}_{1}$, where $\mathcal{P}_{1} \in \mathbf{R}^{\hat{r}\times
\hat{r}}$ and $\mathcal{P}_{2} \in \mathbf{R}^{\hat{s}\times \hat{s}}$ are orthogonal and $\Pi \in \mathbf{R}^{\hat{r}\times
\hat{s}}$ is a diagonal matrix. Then we can prove for $1 \leq l \leq \min\{\hat{r}, \hat{s}\}$ that
\begin{equation}\label{DualSolution}
    \left\{
       \begin{array}{l}
         \mathcal{W}_{x} = \mathcal{U}_{1}\Pi^{-1}_{1}\mathcal{P}_{1}(:,1:l) + \mathcal{U}_{2}\mathcal{E},\\
         \mathcal{W}_{y} = \mathcal{V}_{1}\Pi^{-1}_{2}\mathcal{P}_{2}(:,1:l) + \mathcal{V}_{2}\mathcal{F},
       \end{array}
    \right.
\end{equation}
with $\mathcal{E} \in \mathbf{R}^{(n-\hat{r})\times l}$ and $\mathcal{F} \in \mathbf{R}^{(n-\hat{s})\times l}$ being arbitrary
matrices, form a subset of solutions to \eqref{KCCA}.

Solutions of \eqref{KCCA} can also be associated with least squares problems. Define
\begin{align}
  \mathcal{T}_{x} &= K_{x}K^{\dagger}_{x}\mathcal{V}_{1}\mathcal{P}_{2}(:,1:l)(\Pi(1:l,1:l))^{-1} = \mathcal{U}_{1}\mathcal{P}_{1}(:,1:l), \label{DualPredictorX} \\
  \mathcal{T}_{y} &= K_{y}K^{\dagger}_{y}\mathcal{U}_{1}\mathcal{P}_{1}(:,1:l)(\Pi(1:l,1:l))^{-1} =
\mathcal{V}_{1}\mathcal{P}_{2}(:,1:l), \label{DualPredictorY}
\end{align}
with $1 \leq l \leq \hat{m}$, then each pair of $\mathcal{W}_{x}$ and $\mathcal{W}_{y}$, satisfying
\begin{equation*}
    \mathcal{W}_{x} = \text{arg}\min\{\|K_{x}\mathcal{W}_{x} - \mathcal{T}_{x}\|^{2}_{F}: \mathcal{W}_{x} \in \mathbf{R}^{n\times
l}\},
\end{equation*}
and
\begin{equation*}
    \mathcal{W}_{y} = \text{arg}\min\{\|K_{y}\mathcal{W}_{y} - \mathcal{T}_{y}\|^{2}_{F}: \mathcal{W}_{y} \in \mathbf{R}^{n\times
l}\},
\end{equation*}
respectively, form a solution of \eqref{KCCA}.

Similar to the derivation of sparse CCA in Section \ref{LSCCA}, we incorporate sparsity into $\mathcal{W}_{x}$ and
$\mathcal{W}_{y}$ through solving the following $\ell_{1}$-norm regularized least squares problems
\begin{eqnarray}
  \min\{\frac{1}{2}\|K_{x}\mathcal{W}_{x} - \mathcal{T}_{x}\|^{2}_{F} + \sum\limits^{l}_{i=1}\rho_{x,i} \|\mathcal{W}_{x,i}\|_{1}:
  \mathcal{W}_{x} \in \mathbf{R}^{n\times l}\}, \label{SparseKCCAx}\\
  \min\{\frac{1}{2}\|K_{y}\mathcal{W}_{y} - \mathcal{T}_{y}\|^{2}_{F} + \sum\limits^{l}_{i=1}\rho_{y,i} \|\mathcal{W}_{y,i}\|_{1}:
  \mathcal{W}_{y} \in \mathbf{R}^{n\times l}\}, \label{SparseKCCAy}
\end{eqnarray}
where $\rho_{x,i},\rho_{y,i} > 0$ are regularization parameters. Applying fixed-point iterative method to \eqref{SparseKCCAx} and
\eqref{SparseKCCAy}, we get a new sparse kernel CCA algorithm presented in Algorithm \ref{SKCCAFPC}

\begin{algorithm}[H]
\caption{(SKCCA: Sparse kernel CCA)} \label{SKCCAFPC}
\begin{algorithmic}[1]
\REQUIRE Training data $X \in \mathbf{R}^{d_{1}\times n}$, $Y \in \mathbf{R}^{d_{2}\times n}$                        % Input

\ENSURE Sparse transformation matrices $\mathcal{W}_{x}\in \mathbf{R}^{n\times l}$ and $\mathcal{W}_{y}\in \mathbf{R}^{n\times
l}$.

\STATE Construct and center kernel matrices $K_{x}$, $K_{y}$;

\STATE Compute matrix factorizations \eqref{KxEigenDec}-\eqref{KxKySVD};

\STATE Compute $\mathcal{T}_{x}$ and $\mathcal{T}_{y}$ defined in \eqref{DualPredictorX}-\eqref{DualPredictorY};

\REPEAT
    \STATE $\mathcal{W}^{k+1}_{x,i} = \mathcal{S}_{\nu_{x,i}}\left(\mathcal{W}^{k}_{x,i} - \tau_{x} K_{x}(K^{T}_{x}\mathcal{W}^{k}_{x,i} -
           \mathcal{T}_{x,i})\right)$, $\nu_{x,i} = \tau_{x}\rho_{x,i}$, $i=1,\cdots,l$,
\UNTIL {convergence}

\REPEAT
    \STATE $\mathcal{W}^{k+1}_{y,i} = \mathcal{S}_{\nu_{y,i}}\left(\mathcal{W}^{k}_{y,i} - \tau_{y} K_{y}(K^{T}_{y}\mathcal{W}^{k}_{y,i} -
           \mathcal{T}_{y,i})\right)$, $\nu_{y,i} = \tau_{y}\rho_{y,i}$, $i=1,\cdots,l$,
\UNTIL {convergence}

\RETURN $\mathcal{W}_{x} = \mathcal{W}^{k}_{x}$ and $\mathcal{W}_{y} = \mathcal{W}^{k}_{y}$.
\end{algorithmic}
\end{algorithm}

Since canonical correlations in kernel CCA depend only on kernel matrices $K_{x}$ and $K_{y}$. Therefore, as we shall see from
factorizations \eqref{KxEigenDec}-\eqref{KxKySVD}, canonical correlations in kernel CCA are determined by singular values of
$\mathcal{U}^T_{1}\mathcal{V}_{1}$. The following lemma reveals a simple result regarding the distribution of canonical
correlations.

\begin{lemma}\label{CanCorr}
Let $\hat{r} = \mbox{rank}(K_{x})$ and $\hat{s} = \mbox{rank}(K_{y})$. If $\hat{r} + \hat{s} = n + \gamma$ for some $\gamma
> 0$, then $\mathcal{U}^T_{1}\mathcal{V}_{1}$ has at least $\gamma$ singular values equal to 1.
\end{lemma}
\begin{proof}
Since $\mathcal{U}_{1} \in \mathbf{R}^{n\times \hat{r}}$, $\mathcal{U}_{2} \in \mathbf{R}^{n\times (n-\hat{r})}$ and
$\mathcal{V}_{1} \in \mathbf{R}^{n\times \hat{s}}$ are column orthogonal and $\mathcal{U}_{1}\mathcal{U}^{T}_{1} +
\mathcal{U}_{2}\mathcal{U}^{T}_{2} = I_{n}$, we have
$$
(\mathcal{U}^T_{1}\mathcal{V}_{1})^{T}\mathcal{U}^T_{1}\mathcal{V}_{1} =
\mathcal{V}^T_{1}\mathcal{U}_{1}\mathcal{U}^T_{1}\mathcal{V}_{1} = I_{\hat{s}} -
\mathcal{V}^T_{1}\mathcal{U}_{2}\mathcal{U}^T_{2}\mathcal{V}_{1}.
$$
If there exist $\gamma > 0$ such that $\hat{r} + \hat{s} = n + \gamma$, then $n-\hat{r} = \hat{s}-\gamma < \hat{s}$ and
$$
\text{rank}(\mathcal{V}^T_{1}\mathcal{U}_{2}\mathcal{U}^T_{2}\mathcal{V}_{1}) = \text{rank}(\mathcal{U}^T_{2}\mathcal{V}_{1})
\leq n-\hat{r},
$$
which implies $\mathcal{V}^T_{1}\mathcal{U}_{2}\mathcal{U}^T_{2}\mathcal{V}_{1}$ has at least $\hat{s}-(n-\hat{r})=\gamma$ zero
eigenvalues. Thus, $(\mathcal{U}^T_{1}\mathcal{V}_{1})^{T}\mathcal{U}^T_{1}\mathcal{V}_{1}$ has at least $\gamma$ eigenvalues
equal to 1, which further implies that $\mathcal{U}^T_{1}\mathcal{V}_{1}$ has at least $\gamma$ singular values equal to 1.
\end{proof}

A direct result of lemma \ref{CanCorr} is that there are at least $\gamma$ canonical correlations in kernel CCA are 1. In kernel
methods, due to nonlinearity of kernel functions the rank of kernel matrices is very close to $n$, which makes most canonical
correlations to be 1. For example, polynomial kernel and Gaussian kernel
\begin{equation}\label{Polykernel}
    \kappa(x,y)=(\gamma_{1}(x\cdot y)+\gamma_{2})^{d},~d>0, \text{~and~}\gamma_{1},\gamma_{2}\in \mathbf{R},
\end{equation}
\begin{equation}\label{Gausskernel}
    \kappa(x,y)=\text{exp}\left(-\frac{1}{2\sigma^{2}}\|x-y\|^{2}\right),
\end{equation}
are two widely used kernel functions. For Gaussian kernel we can prove that if $\sigma \neq 0$, then the kernel matrix $K_{x}$
given by
$$
(K_{x})_{ij}=\text{exp}\left(-\frac{1}{2\sigma^{2}}\|x_{i}-x_{j}\|^{2}\right)
$$
has full rank, given the points $\{x_{i}\}^{n}_{i=1}$ are distinct. A similar result can be proven for linear kernel
\begin{equation}\label{Linearkernel}
    \kappa(x,y)=x\cdot y,
\end{equation}
which is a special case of polynomial kernel \eqref{Polykernel}, when $\{x_{i}\}^{n}_{i=1}$ and $\{y_{i}\}^{n}_{i=1}$ are
linearly independent, respectively. Thus, in kernel methods we usually have
$$
\hat{r} = \text{rank}(K_{x}) = n-1, \quad \hat{s} = \text{rank}(K_{y}) = n-1,
$$
after centering data. Since $K_{x}e = 0$ and $K_{x}e = 0$, we see that $\mathcal{U}^{T}_{1}e=\mathcal{V}^{T}_{1}e=0$ and both
$\begin{bmatrix}
    \mathcal{U}_{1} & \frac{e}{\sqrt{n}} \\
 \end{bmatrix}$ and
$\begin{bmatrix}
    \mathcal{V}_{1} & \frac{e}{\sqrt{n}} \\
 \end{bmatrix}$
are orthogonal matrices. This implies that
$$
\begin{bmatrix}
    \mathcal{U}_{1} & \frac{e}{\sqrt{n}} \\
 \end{bmatrix}^{T}
\begin{bmatrix}
    \mathcal{V}_{1} & \frac{e}{\sqrt{n}} \\
 \end{bmatrix}=
\begin{bmatrix}
    \mathcal{U}^{T}_{1}\mathcal{V}_{1} & 0 \\
     0 & 1
 \end{bmatrix}
$$
is orthogonal. In this case, all nonzero canonical correlations determined by the singular values of
$\mathcal{U}^{T}_{1}\mathcal{V}_{1}$ are equal to 1. Therefore, ordinary kernel CCA fails to provide a useful estimation of
canonical correlations for general kernels, because for any distinct sample $\{x_{i}\}^{n}_{i=1}$ of variable $x$ and distinct
sample $\{y_{i}\}^{n}_{i=1}$ of variable $y$ the canonical correlations returned by kernel CCA will be 1 even though variables
$x$ and $y$ have no joint information.

To avoid aforementioned data overfitting problem in kernel CCA, researchers suggested to design a regularized kernelization of CCA
\cite{Bach03,Bie05,Fukumizu07,Hardoon04,Kuss03}. One way of regularization is to penalize weight vectors $w_{x}$ and $w_{y}$,
leading to
$$
\eta = \max_{\alpha,\beta} \frac{\alpha^{T}K_{x}K_{y}\beta}{\sqrt{(\alpha^{T}K^{2}_{x}\alpha +
\rho_{x}\alpha^{T}K_{x}\alpha)(\beta^{T}K^{2}_{y}\beta + \rho_{y}\beta^{T}K_{y}\beta)}}.
$$
As shown in \cite{Bie05}, dual vectors $\alpha$ and $\beta$ solving the above optimization problem form an eigenvector of
\begin{equation}\label{RKCCA}
\begin{bmatrix}
  0 & K_{x}K_{y} \\
  K_{y}K_{x} & 0 \\
\end{bmatrix}\begin{bmatrix}
               \alpha \\
               \beta \\
             \end{bmatrix}=\eta
\begin{bmatrix}
  K^{2}_{x}+\rho_{x}K_{x} & 0 \\
  0 & K^{2}_{y}+\rho_{y}K_{y} \\
\end{bmatrix}\begin{bmatrix}
               \alpha \\
               \beta \\
             \end{bmatrix}
\end{equation}
corresponding to the largest eigenvalue. Dual transformation matrices $\mathcal{W}_{x}, ~\mathcal{W}_{y} \in \mathbf{R}^{n\times
l}$ can be obtained by computing eigenvectors corresponding to $l$ leading eigenvalues of \eqref{RKCCA}. If we have the following
SVD
$$
(\Pi_{1}+\rho_{x}I)^{-1/2}\Pi^{1/2}_{1}\mathcal{U}^{T}_{1}\mathcal{V}_{1}\Pi^{1/2}_{2}(\Pi_{2}+\rho_{y}I)^{-1/2} =
\mathcal{Q}_{1}\widetilde{\Pi}\mathcal{Q}^{T}_{2},
$$
where $\mathcal{Q}_{1} \in \mathbf{R}^{\hat{r}\times\hat{r}}$ and $\mathcal{Q}_{2} \in \mathbf{R}^{\hat{s}\times\hat{s}}$ are
orthogonal, then we can use
\begin{equation}\label{RKCCAsolution}
   \left\{
   \begin{array}{l}
    \mathcal{W}_{x} =\mathcal{U}_{1}(\Pi^{2}_{1}+\rho_{x}\Pi_{1})^{-1/2}\mathcal{Q}_{1}(:,1:l),\\
    \mathcal{W}_{y} =\mathcal{V}_{1}(\Pi^{2}_{2}+\rho_{y}\Pi_{2})^{-1/2}\mathcal{Q}_{2}(:,1:l),
   \end{array}\right. \quad
   1\leq l \leq \hat{m},
\end{equation}
as a solution of regularized kernel CCA.

As shown in \cite{Tibshirani96}, the $\ell_{1}$-penalization term can alleviate data overfitting problem while at the same time
introduce sparsity. We can expect that sparse kernel CCA \eqref{SparseKCCAx}-\eqref{SparseKCCAy} enjoys the properties of both
computing sparse $\mathcal{W}_{x}$ and $\mathcal{W}_{y}$ and avoiding data overfitting similar to regularized kernel CCA.

\section{Numerical results}\label{experiments}
\setcounter{equation}{0}

In this section, we implement the proposed sparse CCA and sparse kernel CCA algorithms, referred to as SCCA$\_$LS and SKCCA,
respectively, on both artificial and real data. In section 5.1, we describe experimental settings, including stoping criteria of
our algorithms and regularization parameter choices. In section 5.2, we apply SCCA$\_$LS to dimension reduction and
classification, and compare it with ordinary CCA and SCCA$\_\ell_{1}$ \cite{ChuLNZ13}---a sparse CCA algorithm based on
$\ell_{1}$-minimization. In section 5.3, we apply both ordinary CCA and kernel CCA (KCCA) to artificial data, which illustrates
the advantage of kernel CCA over ordinary CCA in finding nonlinear relationship. In section 5.4, we compare SKCCA with KCCA and
regularized KCCA (RKCCA) in cross-language documents retrieval task. In section 5.5, we compare SKCCA with KCCA and regularized
KCCA (RKCCA) in content based image retrieval. All experiments were performed under CentOS 5.2 and \texttt{MATLAB} v7.4 (R2007a)
running on a IBM HS21XM Bladeserver with two Intel Xeon E5450 3.0GHz quad-core Harpertown CPUs and 16GB of Random-access memory
(RAM).

\subsection{Experimental settings}
In the implementation of SCCA$\_$LS and SKCCA, we need to determine regularization parameters $\{\lambda_{x,i}\}$ and
$\{\lambda_{y,i}\}$ for SCCA$\_$LS, and $\{\rho_{x,i}\}$ and $\{\rho_{y,i}\}$ for SKCCA. Whe applying SCCA$\_$LS to dimension
reduction and classification in section 5.2, we let
$$\lambda_{x,i}=\lambda_{y,i}=\lambda,~i=1,\cdots,l,$$
for the sake of simplicity. The 5-fold cross-validation was used to choose the optimal $\lambda$ from the candidate set
$\{10^{-4}, 10^{-3}, 10^{-2}, 10^{-1}\}$. When implementing SKCCA in sections 5.4 and 5.5, we selected parameters
$\{\rho_{x,i}\}$ and $\{\rho_{y,i}\}$ in a more subtle way. Since we know that $x^{\ast}$ is a solution of denoising basis
pursuit problem \eqref{L1problem} if and only if
$$
0 \in A^{T}(Ax^{\ast}-b) + \lambda \partial\|x^{\ast}\|_{1},
$$
which implies that $x=0$ is a solution of \eqref{L1problem} when $\lambda \geq \|A^{T}b\|_{\infty}$. To avoid zero solution,
which is meaningless in practice, we chose
$$
\rho_{x,i} = \gamma_{x}\|K^{T}_{x}\mathcal{T}_{x,i}\|_{\infty}, \quad \rho_{y,i} =
\gamma_{y}\|K^{T}_{y}\mathcal{T}_{y,i}\|_{\infty}, \quad i=1,\cdots,l,
$$
where $0 < \gamma_{x} <1$, and $0 < \gamma_{y} <1$.

To perform fixed-point iteration, we use FPC$\_$BB\footnote{\url{http://www.caam.rice.edu/~optimization/L1/fpc/}} algorithm with
\texttt{xtol}=$10^{-5}$ and \texttt{mxitr}=$10^{4}$ and all other parameters default.

In the implementation of RKCCA \eqref{RKCCAsolution}, we also use 5-fold cross-validation to choose an optimal regularization
parameter $\rho_{x}=\rho_{y}=\rho$ from the candidate set $\{10^{-4}, 10^{-3},\cdots, 10^{3}, 10^{4}\}$.

\subsection{Sparse CCA for dimension reduction and classification}

From Section \ref{background} we know that ULDA can be considered as a special case of CCA. In this section, we evaluate the
classification performance of Algorithm \ref{SCCAFPC} on datasets from face image, micro-array and text document databases. Table
\ref{Datasets} describes detailed information of the datasets in our experiments.~\footnote{Gene expression datasets are obtained from site \url{http://stat.ethz.ch/~dettling/bagboost.html}. Their preprocessing is fully described in \cite{Dettling04}. MEDLINE can be downloaded from \url{http://www-users.cs.umn.edu/~hpark/data.html}, and all other text document datasets are downloaded from CLUTO at \url{http://glaros.dtc.umn.edu/gkhome/cluto/cluto/download}. The UMIST face data is available at \url{http://www.sheffield.ac.uk/eee/research/iel/research/face}. The YaleB database is the extended Yale Face Database B~\cite{LeeHK05}. The ORL database can be retrieved from \url{http://www.cl.cam.ac.uk/Research/DTG/attarchive:pub/data/attfaces.tar.Z}. The Essex database is available at \url{http://cswww.essex.ac.uk/mv/allfaces/index.html}. The Palmprint database is available at \url{http://www4.comp.polyu.edu.hk/~biometrics/}. The AR face database is available from \url{http://www2.ece.ohio-state.edu/~aleix/ARdatabase.html}. The Feret face database is available at \url{http://www.itl.nist.gov/iad/humanid/feret/feret_master.html}.}
%
%........ \textit{Put data description here} .........

\begin{table}[H]
  \centering{\footnotesize
  \caption{Summary of datasets}\label{Datasets}
  \begin{tabular}{c|c|c|c|c|c}
    \hline
      & dataset & Dimension & Training & Number of classes & Testing \\
    \hline
    \multirow{6}{*}{Gene} & Colon & 2000 & 31 & 2 & 31 \\
     & Leukemia & 3571 & 37 & 2 & 35 \\
     & Prostate & 6033 & 51 & 2 & 51 \\
     & Lymphoma & 4026 & 32 & 3 & 30 \\
     & Srbct & 2307 & 32 & 4 & 31 \\
     & Brain & 5597 & 21 & 5 & 21 \\
    \hline
    \multirow{5}{*}{Text} & MEDLINE & 22095 & 1250 & 5 & 1250 \\
     & Tr23 & 5832 & 104 & 6 & 100 \\
     & Tr41 & 7454 & 442 & 10 & 436 \\
     & Wap & 8460 & 786 & 20 & 774 \\
     & Kla & 21839 & 1173 & 20 & 1167 \\
    \hline
    \multirow{7}{*}{Image} & UMIST & 10304 & 290 & 20 & 285 \\
     & Yale B & 32256 & 1216 & 38 & 1216 \\
     & ORL & 4096 & 200 & 40 & 200 \\
     & Essex & 23800 & 720 & 72 & 720 \\
     & Palmprint & 4096 & 300 & 100 & 300 \\
     & AR & 2250 & 840 & 120 & 840 \\
     & Feret & 6400 & 600 & 200 & 600 \\
    \hline
  \end{tabular}}
\end{table}

To evaluate more comprehensively the efficiency of SCCA$\_$LS, we compared it with ordinary CCA and a recently proposed sparse
CCA algorithm SCCA$\_\ell_{1}$ \cite{ChuLNZ13}, where linearized Bregman method was replaced with an accelerated linearized Bregman
method. When using SCCA$\_\ell_{1}$, we set $\mu_{x}=10$, $\mu_{y}=100$, $\delta=0.9$ and terminated the iteration with tolerance
$\epsilon=10^{-5}$. In Table \ref{Classification}, we recorded the classification accuracy of these three algorithms. The
classification accuracy was computed by employing the K-Nearest-Neighbor (KNN) classifier with $K=1$ in all cases. We also
recorded sparsity of $W_{x}$ showing the ratio of the number of zeros to the number of entries in $W_{x}$, violation of the
orthogonality constraint measured by $Err(W_{x}):=\frac{\|W^{T}_{x}XX^{T}W_{x}-I_{l}\|_{F}}{\sqrt{l}}$, the regularization
parameter $\lambda$ obtained by cross-validation, the number of columns in $W_{x}$ (i.e., $l$) and CPU time in seconds.

\begin{table}[!t]
  \centering{\footnotesize
  \caption{Sparse CCA for classification}\label{Classification}
  \begin{tabular}{c|c|c|c|c|c|c|c}
    \toprule
    Data & Algorithms & Accuracy (\%) & Sparsity (\%) & $Err(W_{x})$ & $\lambda$ & $l$ & CPU time (s)\\
    \hline\hline
    \multirow{3}{*}{Leukemia} & SCCA$\_$LS & 94.286 & 98.4 & 6.770e-5 & $10^{-4}$ & 1 & 7.83 \\
    \cline{2-8}
     & SCCA$\_\ell_{1}$ & 94.286 & 99.0 & 2.719e-8 & - & 1 & 1.35e+2 \\
    \cline{2-8}
     & CCA & 97.143 & 0 & 8.882e-16 & - & 1 & 0.03 \\
    \hline
    \multirow{3}{*}{Colon} & SCCA$\_$LS & 74.194 & 95.5 & 1.173e-4 & $10^{-4}$ & 1 & 0.35 \\
    \cline{2-8}
     & SCCA$\_\ell_{1}$ & 74.194 & 98.5 & 1.123e-5 & - & 1 & 11.32 \\
    \cline{2-8}
     & CCA & 67.742 & 0 & 4.441e-16 & - & 1 & 0.04 \\
    \hline
    \multirow{3}{*}{Prostate} & SCCA$\_$LS & 94.118 & 96.7 & 1.076e-4 & $10^{-4}$ & 1 & 17.23 \\
    \cline{2-8}
     & SCCA$\_\ell_{1}$ & 94.118 & 99.2 & 2.663e-6 & - & 1 & 85.45 \\
    \cline{2-8}
     & CCA & 92.157 & 0 & 1.332e-15 & - & 1 & 0.08 \\
    \hline
    \multirow{3}{*}{Lymphoma} & SCCA$\_$LS & 100 & 97.3 & 6.223e-5 & $10^{-4}$ & 2 & 18.00 \\
    \cline{2-8}
     & SCCA$\_\ell_{1}$ & 100 & 99.2 & 3.069e-6 & - & 2 & 1.03e+2 \\
    \cline{2-8}
     & CCA & 100 & 0 & 4.906e-15 & - & 2 & 0.04 \\
    \hline
    \multirow{3}{*}{Srbct} & SCCA$\_$LS & 93.548 & 96.4 & 4.930e-5 & $10^{-4}$ & 3 & 0.86 \\
    \cline{2-8}
     & SCCA$\_\ell_{1}$ & 96.774 & 98.7 & 3.588e-6 & - & 3 & 5.78e+2 \\
    \cline{2-8}
     & CCA & 96.774 & 0 & 1.929e-15 & - & 3 & 0.03 \\
    \hline
    \multirow{3}{*}{Brain} & SCCA$\_$LS & 76.191 & 99.3 & 7.725e-3 & $10^{-2}$ & 4 & 1.16 \\
    \cline{2-8}
     & SCCA$\_\ell_{1}$ & 76.191 & 99.6 & 9.091e-6 & - & 4 & 2.44e+2 \\
    \cline{2-8}
     & CCA & 76.191 & 0 & 1.485e-15 & - & 4 & 0.03 \\
    \hline\hline
    \multirow{3}{*}{MEDLINE} & SCCA$\_$LS & 92.480 & 99.9 & 0.747 & $10^{-1}$ & 4 & 3.31e+2 \\
    \cline{2-8}
     & SCCA$\_\ell_{1}$ & 72.560 & 94.1 & 1.402e-6 & - & 4 & 1.634e+4 \\
    \cline{2-8}
     & CCA & 74.080 & 31.2 & 2.800e-15 & - & 4 & 1.26e+2 \\
    \hline
    \multirow{3}{*}{Tr23} & SCCA$\_$LS & 77.000 & 99.1 & 0.479 & $10^{-3}$ & 5 & 1.45e+2 \\
    \cline{2-8}
     & SCCA$\_\ell_{1}$ & 71.000 & 98.2 & 2.240e-5 & - & 5 & 3.06e+3 \\
    \cline{2-8}
     & CCA & 73.000 & 1.0 & 3.959e-15 & - & 5 & 0.27 \\
    \hline
    \multirow{3}{*}{Tr41} & SCCA$\_$LS & 89.450 & 95.9 & 0.139 & $10^{-3}$ & 9 & 2.64e+3 \\
    \cline{2-8}
     & SCCA$\_\ell_{1}$ & 85.551 & 94.1 & 8.959e-6 & - & 9 & 1.33e+4 \\
    \cline{2-8}
     & CCA & 88.991 & 3.5 & 4.074e-15 & - & 9 & 5.36 \\
    \hline
    \multirow{3}{*}{Wap} & SCCA$\_$LS & 80.362 & 94.4 & 0.267 & $10^{-1}$ & 19 & 9.10e+2 \\
    \cline{2-8}
     & SCCA$\_\ell_{1}$ & 75.840 & 90.7 & 5.866e-6 & - & 19 & 3.75e+4 \\
    \cline{2-8}
     & CCA & 77.390 & 2.9 & 5.189e-15 & - & 19 & 28.93 \\
    \hline
    \multirow{3}{*}{Kla} & SCCA$\_$LS & 84.576 & 97.3 & 0.284 & $10^{-1}$ & 19 & 3.28e+3 \\
    \cline{2-8}
     & SCCA$\_\ell_{1}$ & 79.520 & 94.6 & 6.729e-6 & - & 19 & 1.37e+5 \\
    \cline{2-8}
     & CCA & 81.405 & 22.1 & 1.633e-14 & - & 19 & 1.27e+2 \\
    \hline\hline
%    \multirow{3}{*}{Yale} & SCCA$\_$LS & 84.444 & 91.2 & 8.948e-3 & $10^{-4}$ & 14 & 2.12e+2 \\
%    \cline{2-8}
%     & SCCA$\_\ell_{1}$ & 85.556 & 98.2 & 2.385e-5 & - & 14 & 6.164e+3 \\
%    \cline{2-8}
%     & CCA & 87.778 & 0 & 3.818e-15 & - & 14 & 0.07 \\
%    \hline
    \multirow{3}{*}{UMIST} & SCCA$\_$LS & 98.246 & 99.1 & 0.511 & $10^{-1}$ & 19 & 6.71e+2 \\
    \cline{2-8}
     & SCCA$\_\ell_{1}$ & 94.386 & 97.2 & 1.880e-5 & - & 19 & 1.64e+4 \\
    \cline{2-8}
     & CCA & 96.842 & 0 & 5.124e-15 & - & 19 & 4.21 \\
    \hline
    \multirow{3}{*}{YaleB} & SCCA$\_$LS & 95.559 & 95.9 & 0.524 & $10^{-1}$ & 37 & 6.37e+3 \\
    \cline{2-8}
     & SCCA$\_\ell_{1}$ & 83.388 & 96.2 & 1.645e-4 & - & 37 & 7.52e+4 \\
    \cline{2-8}
     & CCA & 75.411 & 0 & 2.167e-14 & - & 37 & 1.80e+2 \\
    \hline
    \multirow{3}{*}{ORL} & SCCA$\_$LS & 93.000 & 84.0 & 4.150e-2 & $10^{-4}$ & 39 & 5.83e+4 \\
    \cline{2-8}
     & SCCA$\_\ell_{1}$ & 93.000 & 95.1 & 2.925e-5 & - & 39 & 7.52e+4 \\
    \cline{2-8}
     & CCA & 94.500 & 0 & 6.083e-15 & - & 39 & 1.05 \\
    \hline
    \multirow{3}{*}{Essex} & SCCA$\_$LS & 77.500 & 98.8 & 0.705 & $10^{-1}$ & 71 & 1.01e+4 \\
    \cline{2-8}
     & SCCA$\_\ell_{1}$ & 63.194 & 96.9 & 3.275e-5 & - & 71 & 2.51e+5 \\
    \cline{2-8}
     & CCA & 67.361 & 0 & 7.057e-15 & - & 71 & 58.19 \\
    \hline
    \multirow{3}{*}{Palmprint} & SCCA$\_$LS & 99.000 & 95.5 & 0.404 & $10^{-1}$ & 99 & 3.53e+3 \\
    \cline{2-8}
     & SCCA$\_\ell_{1}$ & 99.000 & 92.7 & 1.801e-5 & - & 99 & 1.22e+4 \\
    \cline{2-8}
     & CCA & 99.333 & 0 & 7.728e-15 & - & 99 & 1.52 \\
    \hline
    \multirow{3}{*}{AR} & SCCA$\_$LS & 88.571 & 91.4 & 0.782 & $10^{-1}$ & 119 & 1.48e+3 \\
    \cline{2-8}
     & SCCA$\_\ell_{1}$ & 77.500 & 62.6 & 5.129e-5 & - & 119 & 1.10e+4 \\
    \cline{2-8}
     & CCA & 81.310 & 0 & 8.550e-15 & - & 119 & 18.03 \\
    \hline
    \multirow{3}{*}{Feret} & SCCA$\_$LS & 70.500 & 94.8 & 0.662 & $10^{-1}$ & 199 & 1.56e+3 \\
    \cline{2-8}
     & SCCA$\_\ell_{1}$ & 62.500 & 90.6 & 3.390e-5 & - & 199 & 5.98e+4 \\
    \cline{2-8}
     & CCA & 70.750 & 0 & 8.86e-15 & - & 199 & 9.66 \\
    \bottomrule
  \end{tabular}}
\end{table}

\subsection{Synthetic data}

In this section, we apply ordinary CCA, RKCCA and SKCCA on synthetic data to demonstrate the ability of kernel CCA in finding
nonlinear relationship. Let $Z$ be a random variable following uniform distribution over interval $(-2,2)$, we sampled 500 pairs
of $(X,Y)$ in the following way:
$$
X=[Z;Z] \quad \text{and} \quad Y=[Z^{2} + 0.3\epsilon_{1};\sin(\pi Z) + 0.3\epsilon_{2}],
$$
where $\epsilon_{1}$ and $\epsilon_{2}$ follow standard normal distribution. Obviously, variables $X$ and $Y$ are nonlinearly
related. We attempt to reveal the nonlinear association using the first pair of canonical variables $w^{T}_{x}X$ and $w^{T}_{y}Y$,
and we also plot $(w^{T}_{x}X, w^{T}_{y}Y)$ in \figurename~\ref{simulation}. When implementing kernel CCA we employed the
Gaussian kernel \eqref{Gausskernel} with $\sigma$ equal to the maximum distance between data points. We set regularization
parameters $\rho_{x}=\rho_{y}=10^{-2}$ in RKCCA and $\rho_{x}=\rho_{y}=10^{-1}$ in SKCCA.

\subfiglabelskip=0pt
\begin{figure}[!htbp]
\centering %
\subfigure[][]{\includegraphics[width=0.3\textwidth]{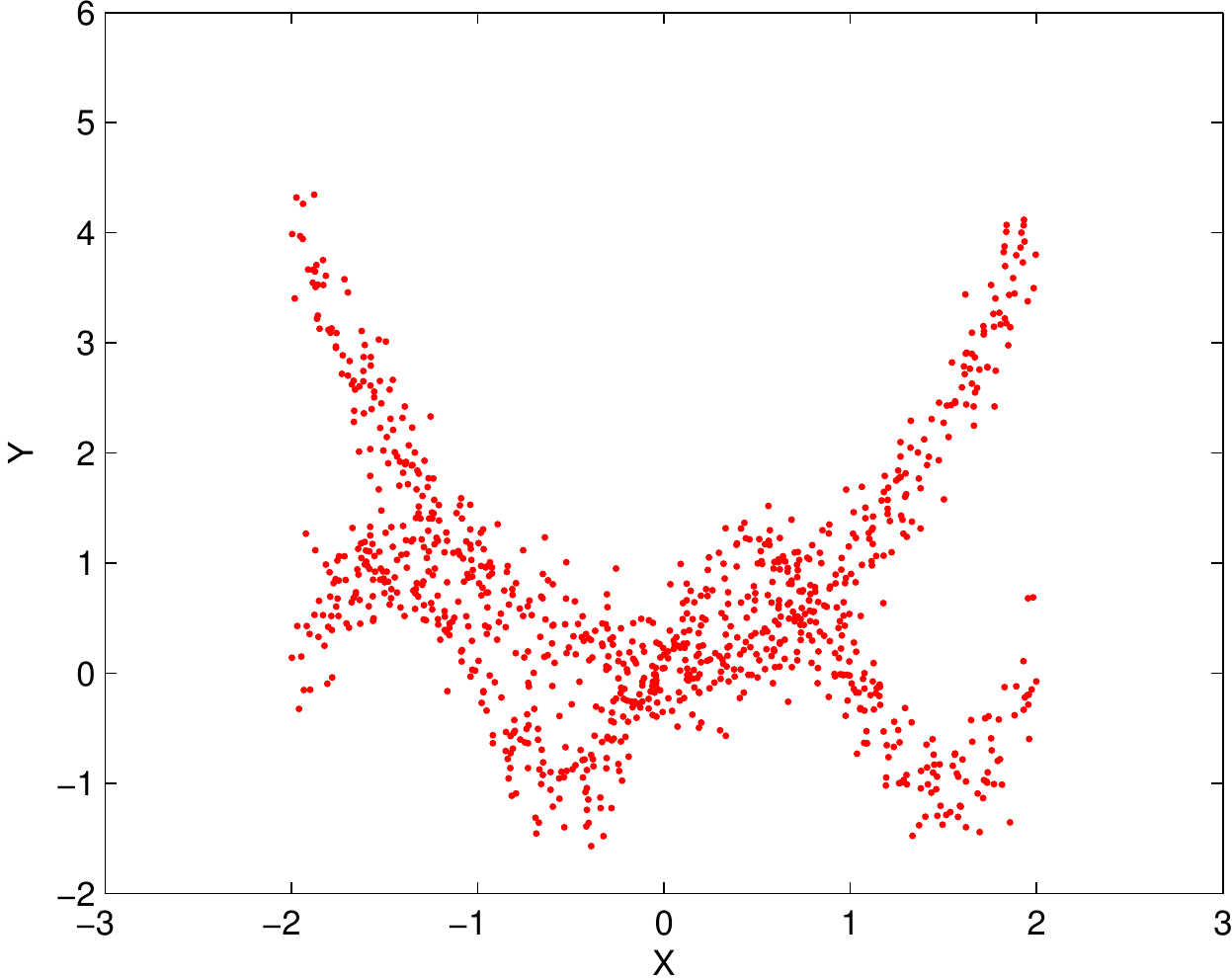} \label{SimData}} \hspace{1pt}
\subfigure[][]{\includegraphics[width=0.3\textwidth]{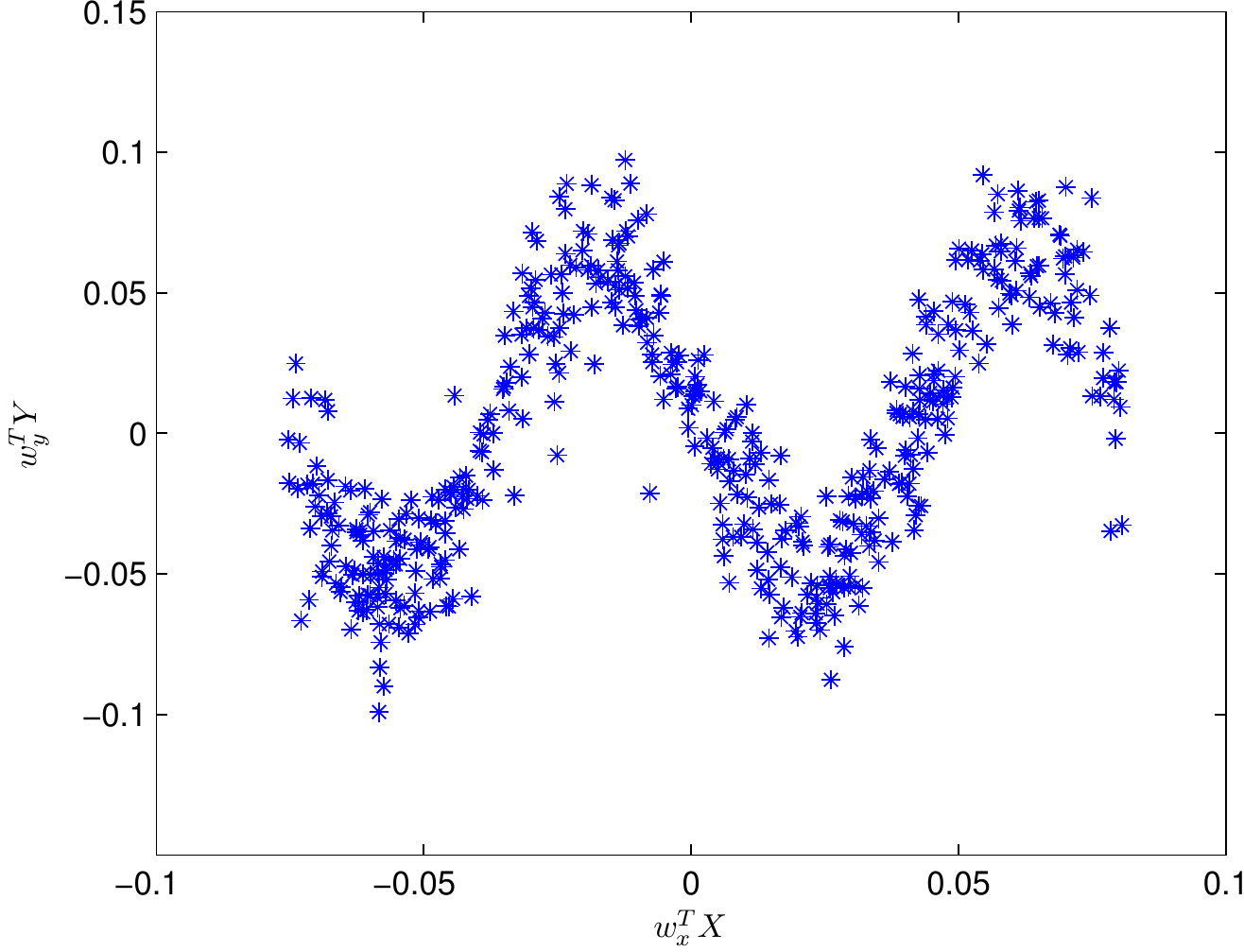} \label{SimCCA}} \\ %
\subfigure[][]{\includegraphics[width=0.3\textwidth]{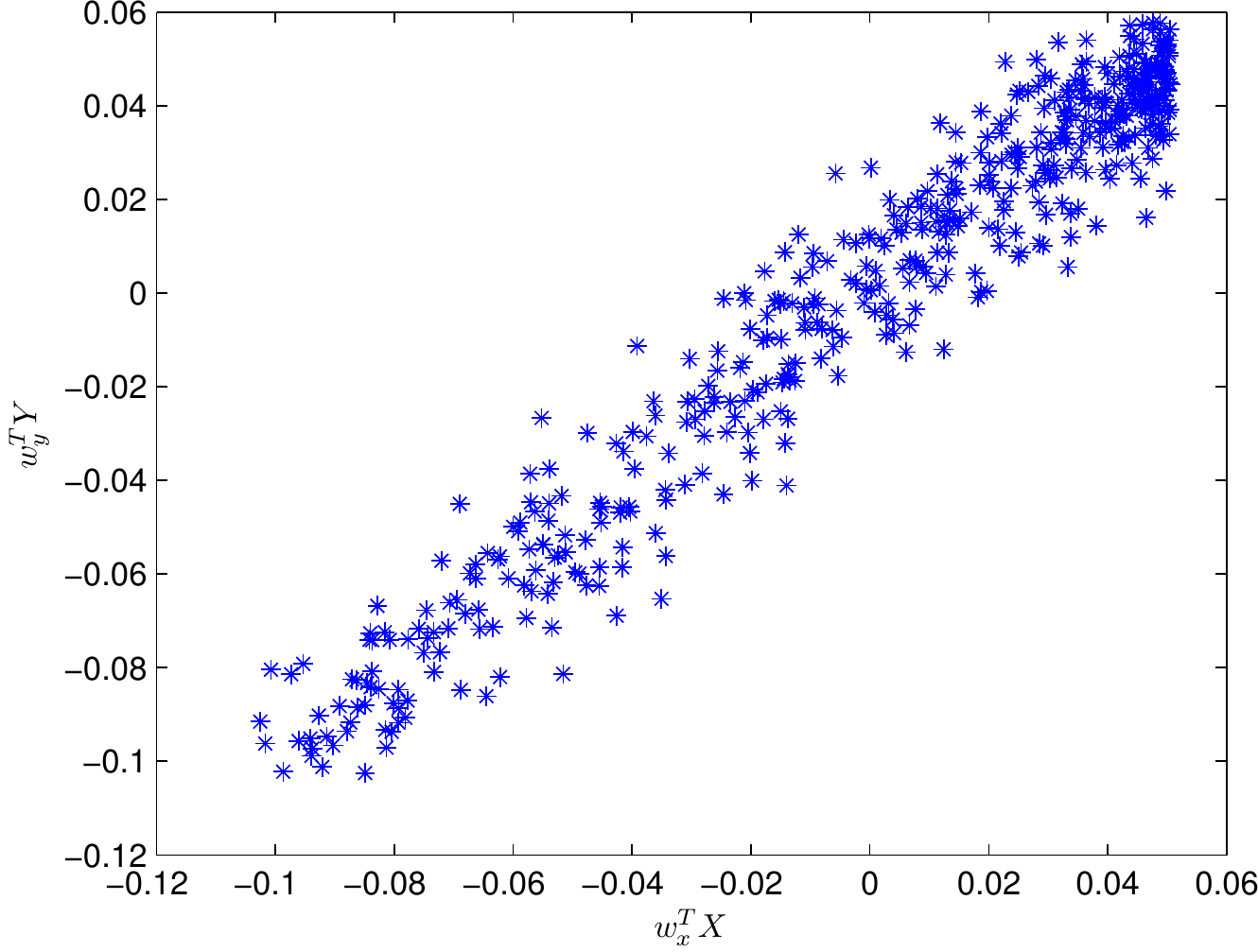} \label{SimRKCCA}} \hspace{1pt}
\subfigure[][]{\includegraphics[width=0.3\textwidth]{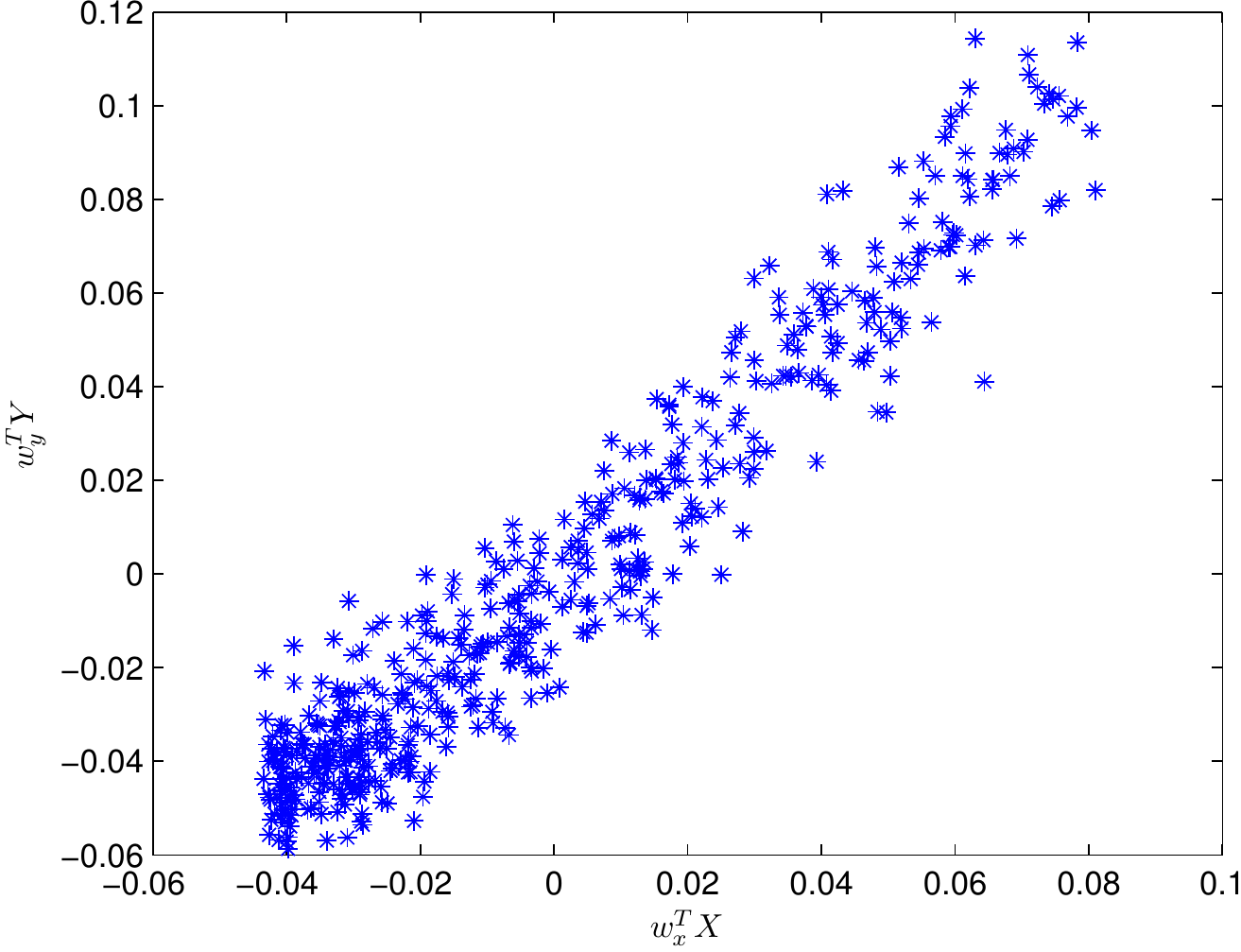} \label{SimSKCCA}} \\ %
\caption{Plots of the first pair of canonical variables:\subref{SimData} sample data, \subref{SimCCA} ordinary CCA, %
\subref{SimRKCCA} RKCCA, \subref{SimSKCCA} SKCCA.}%
\label{simulation}%
\end{figure}

The canonical correlation between the first pair of canonical variables are listed in \tablename~\ref{SimCorr}.

\begin{table}[H]
  \centering
  \begin{tabular}{c|c|c|c}
    \hline
     & CCA & RKCCA & SKCCA \\
    \hline
    canonical correlation & 0.3971 & 0.9621 & 0.9632 \\
    \hline
  \end{tabular}
  \caption{Correlation between the first pair of canonical variables found by ordinary CCA, RKCCA and SKCCA.}\label{SimCorr}
\end{table}

\figurename~\ref{simulation}\subref{SimCCA} shows the data scatter of the first pair of canonical variables found by ordinary CCA,
from which we see that some strong relationship is left unexplained. In contrast, \figurename~\ref{simulation}\subref{SimRKCCA}
and \figurename~\ref{simulation}\subref{SimSKCCA} show the data scatter of the first pair of canonical variables found by RKCCA
and SKCCA, respectively. A clear linear relationship between $w^{T}_{x}X$ and $w^{T}_{y}Y$ can be observed.
\tablename~\ref{SimCorr} also shows that the canonical correlations obtained by RKCCA and SKCCA are 0.9621 and 0.9632,
respectively, which are larger than that achieved by ordinary CCA. The comparison implies that ordinary CCA may not be applicable
to find nonlinear relation of two sets of data.

\subsection{Cross-language document retrieval}

Previous study \cite{Vinokourov03} has shown that kernel CCA works well for cross-language document retrieval and performs better
than the latent semantic indexing approach. In this section, we apply SKCCA to the task of cross-language document retrieval, and
present comparison results of SKCCA with KCCA and RKCCA.

In this experiment, we used the following two datasets:
\begin{enumerate}
  \item The English-French corpus from the Europarl parallel corpus dataset
\cite{Koehn05}\footnote{\url{http://www.statmt.org/europarl/}}, where we obtained 202 samples and generated a $23308\times 202$
term-document matrix for English corpus and a $33986\times 202$ term-document matrix for French corpus.
  \item The Aligned Hansards of the 36th Parliament of Canada \cite{Germann01}\footnote{\url{http://www.isi.edu/natural-language/download/hansard/}},
which is a collection of text chunks (sentences or smaller fragments) in English and French from the 36th Parliament proceedings
of Canada. In our experiments, we used only a part of text chunks to generate term-documents matrices and obtained a $5383\times
818$ term-document matrix for English documents and a $8015\times 818$ term-document matrix for French documents.
\end{enumerate}
For Europarl data 100 pairs of documents were used for training data and the rest for testing data while for Hansard data 200
pairs of documents were used for training data and the rest for testing data. In both experiments, the linear kernel
\eqref{Linearkernel} was employed to compute kernel matrices. We measure the precision of document retrieval by using average
area under the ROC curve (AROC), and for a collection of queries we use the average of each query¡¯s retrieval precision as the
average retrieval precision of this collection. More details about the acquisition of term-document matrix, data preprocessing
and evaluation of retrieval performance can be found in \cite{ChuLNZ13}.

\figurename~\ref{DocRetrievalfig} presents the retrieval accuracy of KCCA, RKCCA and SKCCA on both datasets.

\subfiglabelskip=0pt
\begin{figure}[!htbp]
\centering %
\subfigure[][]{\includegraphics[width=0.45\textwidth]{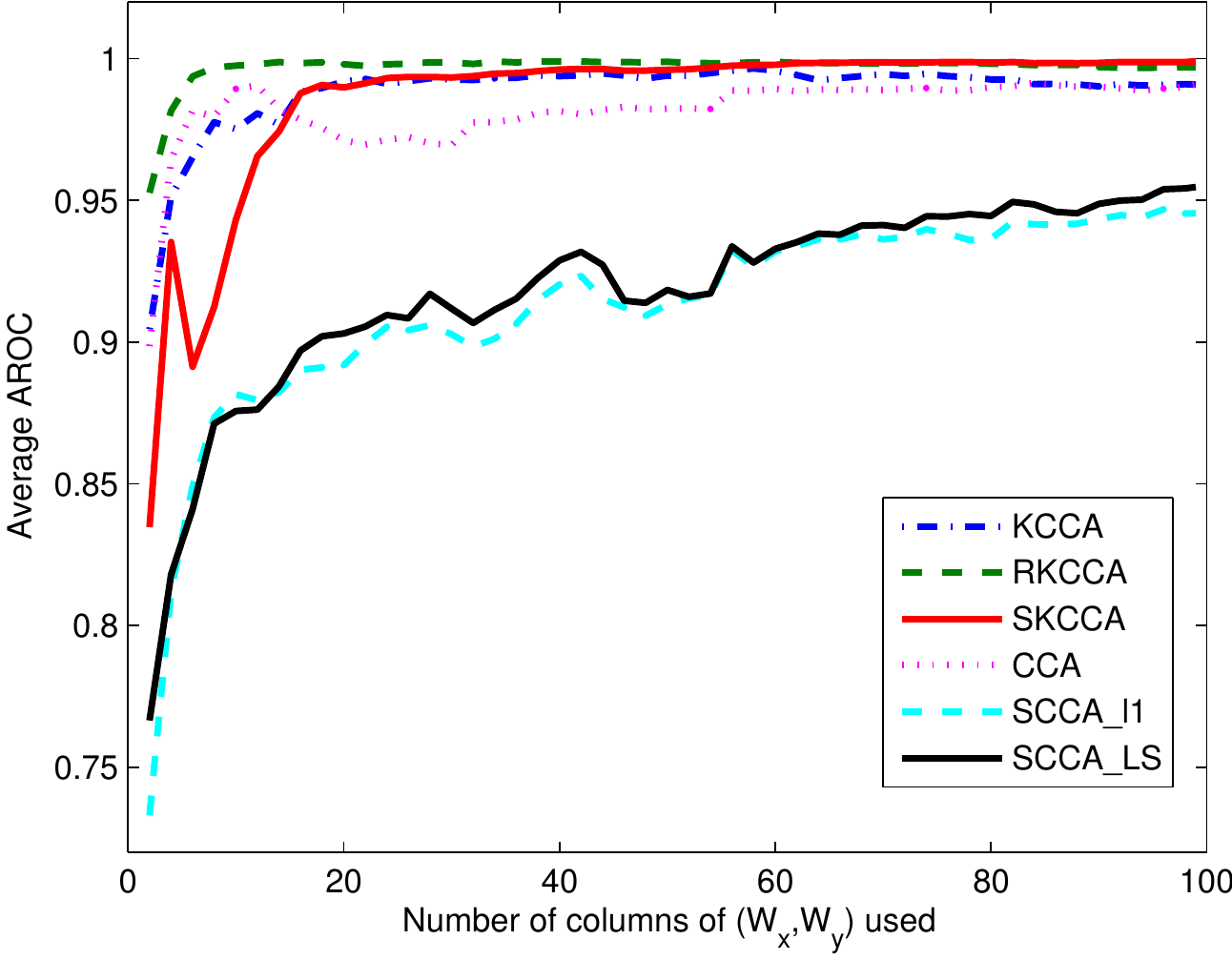} \label{Europarl}} \hfil
\subfigure[][]{\includegraphics[width=0.45\textwidth]{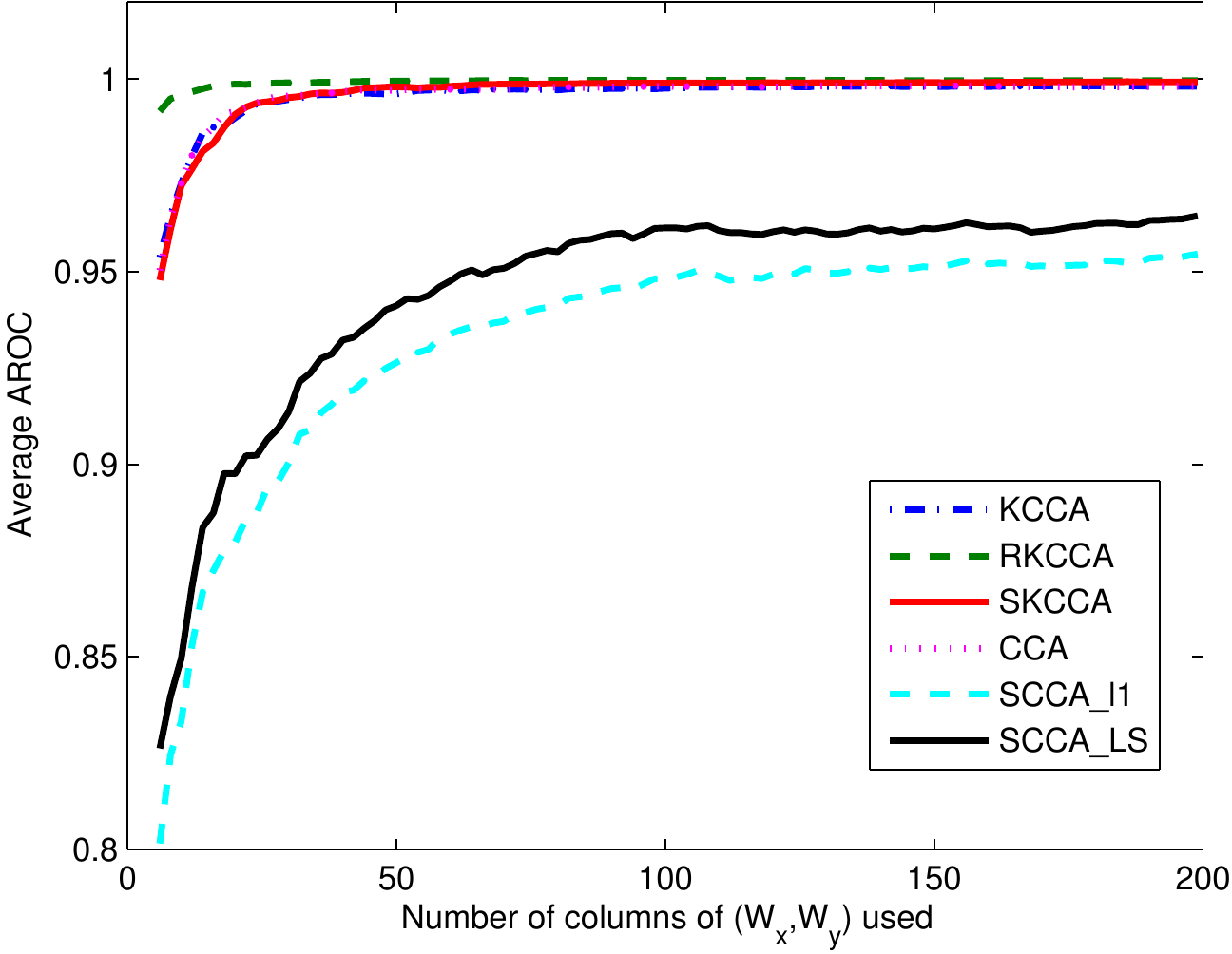} \label{Hansard200}} %
\caption{Cross-language document retrieval using KCCA, RKCCA and SKCCA:
\subref{Europarl} Europarl data with 100 training data, \subref{Hansard200} Hansard data with 200training data.}%
\label{DocRetrievalfig}%
\end{figure}
%
%In addition to the results presented in \figurename~ref{CCAKCCA}, we also experimented on regularized kernel CCA and sparse
%kernel CCA (SKCCA) in two experiments with different number of training data, the results are presented in
%\figurename~\ref{DocRetrieval}.
%
%\subfiglabelskip=0pt
%\begin{figure}[!htbp]
%\centering %
%\subfigure[][]{\includegraphics[width=0.45\textwidth]{CCA_Hansard200} \label{Train200}} \hfil
%\subfigure[][]{\includegraphics[width=0.45\textwidth]{CCA_Hansard600} \label{Train600}} %
%\caption{Cross-language document retrieval using ordinary kernel CCA, regularized kernel CCA and sparse kernel CCA:
%\subref{Train200} 200 pairs of documents for training data, \subref{Train600} 600 pairs of documents for training data.}%
%\label{DocRetrieval}%
%\end{figure}

\figurename~\ref{DocRetrievalfig} shows that all three algorithms achieve high precision for cross-language document retrieval,
even though only a small number of training data were used. From the figures we also see that increasing $l$, the number of
columns of $W_{x}$ and $W_{x}$ used in retrieval task, will usually assist in improving the precision. One possible explanation
may be that when we increase $l$, more projections corresponding to nonzero canonical correlations are used for document
retrieval and these added projections may carry information contained in the training data. Both figures in
\figurename~\ref{DocRetrievalfig} show that RKCCA and SKCCA outperform KCCA in terms of retrieval accuracy, though their
difference is small. This indicates that both RKCCA and SKCCA have ability of avoiding data overfitting problem in ordinary
kernel CCA, as stated in Section \ref{KernelCCA}.

%An interesting observation is that the precision of ordinary kernel CCA in the experiment with 600 training data is very low (a
%little more than 0.5), which is a tremendous decrease compared with the experiment with 200 training data. This is a consequence
%of data overfitting which results in dual projections $\mathcal{W}_{x}$ and $\mathcal{W}_{y}$ containing little information of
%the data.

Additional results are presented in the following table, where we recorded the retrieval precision (AROC) using projections
corresponding to all nonzero canonical correlations, i.e, $l=\hat{m}$, summation of canonical correlations between testing data
($Corr$), sparsity of $\mathcal{W}_{x}$ and $\mathcal{W}_{y}$, and violation of the orthogonality constraints measured by
$Err(\mathcal{W}_{x}):=\frac{\|\mathcal{W}^{T}_{x}K^{2}_{x}\mathcal{W}_{x}-I_{l}\|_{F}}{\sqrt{l}}$ and
$Err(\mathcal{W}_{y}):=\frac{\|\mathcal{W}^{T}_{y}K^{2}_{y}\mathcal{W}_{y}-I_{l}\|_{F}}{\sqrt{l}}$.

\begin{table}[!t]
  \centering{\footnotesize
  \caption{Ordinary, regularized and sparse kernel CCA for cross-language document retrieval}\label{DocumentRetrieval}
  \begin{tabular}{c|c|c|c|c|c|c|c|c}
    \toprule
    Algorithms & AROC & $Corr$ & Sparsity & $Err(\mathcal{W}_{x})$ & $Err(\mathcal{W}_{y})$ & $l$ & $(\gamma_{x}, \gamma_{y})$ or $\rho$ & CPU time (s) \\
    \hline\hline
    \multicolumn{9}{c}{Europarl: 100 training data}\\
    \hline
    CCA & 0.9910 & 78.99 & (3.7, 1.7) & 1.234e-14 & 1.091e-14 & 99 & - & 4.30 \\
    \hline
    SCCA$\_\ell_{1}$ & 0.9455 & 61.34 & (99.5, 99.7) & 1.467e-5 & 1.477e-5 & 99 & - & 1.48e+4 \\
    \hline
    SCCA$\_$LS & 0.9547 & 63.14 & (99.7, 99.8) & 0.9965 & 0.9976 & 99 & $10^{-2}$ & 2.08e+3 \\
    \hline
    KCCA & 0.9910 & 78.64 & (0, 0) & 3.190e-15 & 3.291e-15 & 99 & - & 0.43 \\
    \hline
    RKCCA & 0.9970 & 77.29 & (0, 0) & 0.5456 & 0.5542 & 99 & $10^{0}$ & 0.41 \\
    \hline
    SKCCA & 0.9991 & 82.60 & (93.8, 93.8) & 0.9874 & 0.9881 & 99 & (0.7, 0.7) & 4.02 \\
    \hline\hline
    \multicolumn{9}{c}{Hansard: 200 training data}\\
    \hline
    CCA & 0.9981 & 160.32 & (12.6, 13.1) & 8.762e-15 & 1.006e-14 & 199 & - & 1.97 \\
    \hline
    SCCA$\_\ell_{1}$ & 0.9546 & 99.54 & (96.1, 97.4) & 1.583e-5 & 1.517e-5 & 199 & - & 1.11e+4 \\
    \hline
    SCCA$\_$LS & 0.9645 & 109.32 & (97.1, 98.1) & 0.4379 & 0.4567 & 199 & $10^{-2}$ & 1.31e+3 \\
    \hline
    KCCA & 0.9981 & 160.01 & (0, 0) & 6.274e-15 & 5.627e-15 & 199 & - & 0.42 \\
    \hline
    RKCCA & 0.9996 & 157.58 & (0, 0) & 0.9905 & 0.9904 & 199 & $10^{2}$ & 0.42 \\
    \hline
    SKCCA & 0.9994 & 166.69 & (89.0, 88.6) & 0.9616 & 0.9603 & 199 & (0.5, 0.5) & 4.21 \\
%    \hline\hline
%    \multicolumn{9}{c}{Hansard: 600 training data}\\
%    \hline
%    KCCA & 0.5465 & 5.31 & (0, 0) & 3.693e-14 & 2.543e-8 & 597 & - & 5.34 \\
%    \hline
%    RKCCA & 0.9993 & 394.92 & (0, 0) & 0.9202 & 0.9186 & 597 & $10^{1}$ & 9.06 \\
%    \hline
%    SKCCA & 0.9998 & 479.48 & (95.0, 94.6) & 0.9843 & 0.9832 & 597 & (0.5, 0.5) & 51.24 \\
    \bottomrule
  \end{tabular}}
\end{table}

\begin{remark}
%The regularization parameters $\lambda$ and $\rho$ were determined by 5-fold cross-validation in the experiment with 200 training
%data, and we used the determined values in the experiment with 600 training data for the sake of fairness when comparing these
%two experiments.

In Table \ref{DocumentRetrieval}, the 'Sparsity' column records sparsity of both $\mathcal{W}_{x}$ and $\mathcal{W}_{y}$. The
first component records sparsity of $\mathcal{W}_{x}$ while the second component records sparsity of $\mathcal{W}_{y}$. The
'$(\gamma_{x}, \gamma_{y})$ or $\rho$' column records value of regularization parameters in RKCCA and SKCCA.
\end{remark}

As can be seen from Table \ref{DocumentRetrieval}, RKCCA and SKCCA achieve high retrieval precision on both datasets, and these
two approaches have comparable performance in terms of precision which is also shown in \figurename~\ref{DocRetrievalfig}.
%On the other hand, AROC obtained by KCCA is very low in the second experiment. This may be caused by overfitting, since we can
%see that the summation of canonical correlations between testing data is very small, which implies that dual projections
%$\mathcal{W}_{x}$ and $\mathcal{W}_{y}$ computed by KCCA fail to find canonical correlations between testing data.
We also note that SKCCA can obtain larger summation of canonical correlations between testing data than other two approaches. In
both experiments sparsity of $\mathcal{W}_{x}$ and $\mathcal{W}_{y}$ computed by KCCA and RKCCA is 0, which means the dual
projections are dense; in contrast, sparsity of $\mathcal{W}_{x}$ and $\mathcal{W}_{y}$ computed by SKCCA is greater than $88\%$,
which means that more than $88\%$ entries of both $\mathcal{W}_{x}$ and $\mathcal{W}_{y}$ are zero.

%Comparing retrieval precision of RKCCA and SKCCA in two experiments, we found that retrieval precision of these two approaches
%stays in a high level and nearly unchanged when the number of training data was increased from 200 to 600, which indicates that
%RKCCA and SKCCA can work very well with limited training data.

In addition, from \figurename~\ref{DocRetrievalfig}, we notice that when $l=10$ AROC of RKCCA and SKCCA is already very high and
increasing $l$ will not improve AROC much. Although AROC will increase as we increase $l$, the increment is very small when
$l>10$. Thus, in order to reduce computing time in practice we do not need to compute dual projections corresponding to all
nonzero canonical correlations.

\subsection{Content-based image retrieval}

Content-based image retrieval (CBIR) is a challenging aspect of multimedia analysis and has become popular in past few years.
Generally, CBIR is the problem of searching for digital images in large databases by their visual content (e.g., color, texture,
shape) rather than the metadata such as keywords, labels, and descriptions associated with the images. There exists study
utilizing kernel CCA for image retrieval \cite{Hardoon04}. In this section, we apply our sparse kernel CCA approach to
content-based image retrieval task by combining image and text data.

We experimented on the following two image datasets:
\begin{enumerate}
  \item Ground Truth Image Database\footnote{\url{http://www.cs.washington.edu/research/imagedatabase/groundtruth/}}  created at the University of Washington, which
consists of 21 datasets of outdoor scene images. In our experiment we used 852 images form 19 datasets that have been annotated
with keywords.
  \item Photography image database used in SIMPLIcity\footnote{\url{http://sites.stat.psu.edu/~jiali/index.download.html}} retrieval
system. The database contains 2360 manually annotated images, from which we randomly selected 1000 images in our experiment.
\end{enumerate}
We exploited text features and low-level image features, including color and texture, and applied sparse kernel CCA to perform
image retrieval from text query.

\noindent\textbf{Text Features} Using the bag-of-words approach, same as what we have done in cross-language document retrieval
experiment, to represent the text associated with images. Since each image in the datasets has been annotated with keywords, we
consider terms adjacent to an image as a document. After removing stop-words and stemming, we get a  term-document matrix of size
$189\times 852$ for Ground Truth Image Data and a term-document matrix of size $141\times 1000$ for SIMPLIcity data.

We applied Gabor filters to extract texture features and used HSV (hue-saturation-value) color representation as color features.
To enhance sensitivity to the overall shape, we divided each image into $8\times 8=64$ patches from which texture and color
features were extracted.

\noindent\textbf{Texture Features} The Gabor filters in the spatial domain is given by
\begin{equation}\label{Gabor}
    g_{\lambda\theta\psi\sigma\gamma}(x,y)=\text{exp}\left(-\frac{x'^{2}+\gamma^{2}y'^{2}}{2\sigma^{2}}\right)
    \text{cos}(2\pi\frac{x'}{\lambda}+\psi),
\end{equation}
where $x'=x\text{cos}(\theta)+y\text{sin}(\theta)$, $y'=-x\text{sin}(\theta)+y\text{cos}(\theta)$, $x$ and $y$ specify the
position of a light impulse. In this equation, $\lambda$ represents the wavelength of the cosine factor, $\theta$ represents the
orientation of the normal to the parallel stripes of a Gabor function in degrees, $\psi$ is the phase offset of the cosine factor
in degrees, $\gamma$ is the spatial aspect ratio and $\sigma$ is the standard deviation of the Gaussian. In
\figurename~\ref{GarborImages} the Gabor filter impulse responses used in this experiment are shown. So from each of the 64 image
patches, the Gabor filter can extract 16 texture features, which eventually results in a total of $64\times 16=1024$ features for
each image.

\begin{figure}[!htbp]
  \centering
  \includegraphics[width=0.5\textwidth]{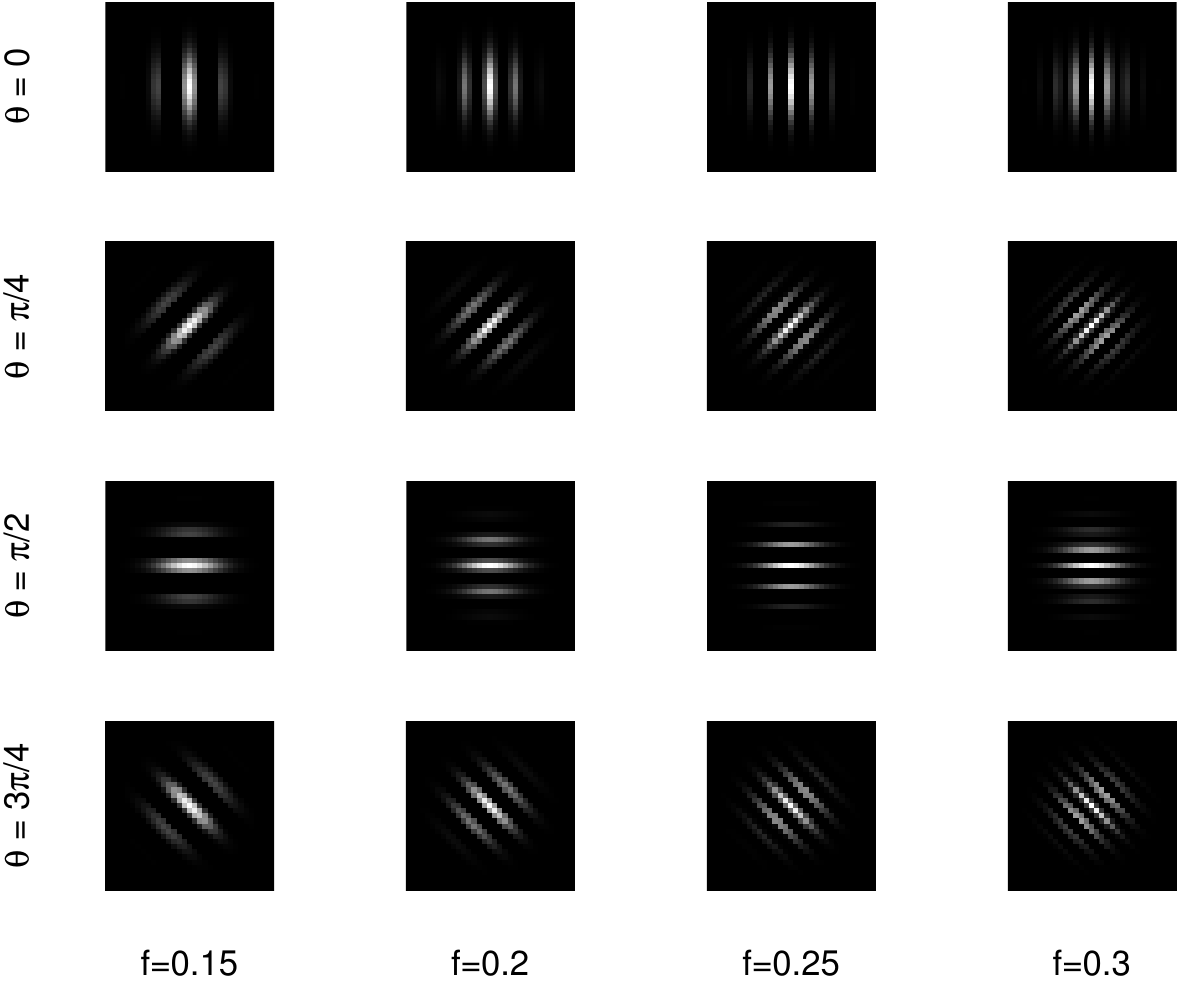}\\
  \caption{Gabor filters used to extract texture features. Four frequencies $f=1/\lambda=[0.15, 0.2, 0.25, 0.3]$
  and four directions $\theta=[0, \pi/4, \pi/2, 3\pi/4]$ are used. The width of the filters are $\sigma=4$.}\label{GarborImages}
\end{figure}

\noindent\textbf{Color Features} We used the HSV color representation as color features. Each color components was quantized into
16 bins, and each image patch was represented by 3 normalized color histograms. This gives 48 features for each of the 64
patches, which eventually results in $48\times 64=3072$ features for each image.

Following previous work \cite{Hardoon04,Hardoon09}, we used Gaussian kernel
$$
k_{x}(I_{i},I_{j})=\text{exp}\left(-\frac{\|I_{i}-I_{j}\|^{2}}{2\sigma^{2}}\right),
$$
where $I_{i}$ is a vector concatenating texture features and color features of $i$th image and $\sigma$ is selected as the
minimum distance between different images, to compute kernel matrix $K_{x}$ for the first view. The linear kernel
\eqref{Linearkernel} was employed to compute kernel matrix $K_{y}$ using text features for the other view. We used 217 images as
training data for the first dataset and 400 images for the second data, the rest were used as testing data.

In \tablename~\ref{CBIR}, we compare the performance of KCCA, RKCCA and SKCCA. Like the cross-language document retrieval
experiments, we use AROC to evaluate the performance of these three algorithms. We see from \tablename~\ref{CBIR} that both RKCCA
and SKCCA outperform KCCA, and RKCCA achieves the best performance in terms of AROC. In both experiments, dual projections
$\mathcal{W}_{x}$ and $\mathcal{W}_{y}$ computed by SKCCA  have high sparsity.

\begin{table}[H]
  \centering{\footnotesize
  \caption{Ordinary, regularized and sparse kernel CCA for Content-based image retrieval}\label{CBIR}
  \begin{tabular}{c|c|c|c|c|c|c|c|c}
    \toprule
    Algorithms & AROC & $Corr$ & Sparsity & $Err(\mathcal{W}_{x})$ & $Err(\mathcal{W}_{y})$ & $l$ & $(\gamma_{x}, \gamma_{y})$ or $\rho$ & CPU time (s) \\
    \hline\hline
    \multicolumn{9}{c}{UW ground truth data: 217 training data}\\
    \hline
    CCA & 0.7396 & 11.53 & (0, 7.7) & 6.832e-15 & 7.638e-15 & 124 & - & 0.50 \\
    \hline
    SCCA$\_\ell_{1}$ & 0.6637 & 7.89 & (94.2, 34.4) & 1.669e-5 & 2.818e-5 & 124 & - & 1.55e+3 \\
    \hline
    SCCA$\_$LS & 0.7140 & 11.99 & (96.5, 43.3) & 0.5420 & 0.3151 & 124 & $10^{-2}$ & 1.61e+2 \\
    \hline
    KCCA & 0.8259 & 19.37 & (0, 0) & 3.215e-015 & 7.069e-14 & 124 & - & 0.39 \\
    \hline
    RKCCA & 0.8912 & 20.76 & (0, 0) & 0.9990 & 0.9984 & 124 & $10^{3}$ & 0.43 \\
    \hline
    SKCCA & 0.8489 & 24.98 & (91.1, 88.4) & 0.9576 & 0.9175 & 124 & (0.5, 0.3) & 7.24 \\
    \hline\hline
    \multicolumn{9}{c}{SIMPLIcity: 400 training data}\\
    \hline
    CCA & 0.7390 & 10.95 & (0, 19.9) & 5.415e-15 & 6.434e-15 & 76 & - & 3.48 \\
    \hline
    SCCA$\_\ell_{1}$ & 0.6509 & 7.58 & (90.2, 42.8) & 1.066e-5 & 2.099e-5 & 76 & - & 4.63e+3 \\
    \hline
    SCCA$\_$LS & 0.7100 & 9.21 & (94.1, 46.7) & 0.5397 & 0.1683 & 76 & $10^{-2}$ & 3.66e+2 \\
    \hline
    KCCA & 0.8509 & 20.21 & (0, 0) & 3.299e-15 & 2.061e-14 & 76 & - & 1.99 \\
    \hline
    RKCCA & 0.8653 & 13.18 & (0, 0) & 0.9989 & 0.9598 & 76 & $10^{2}$ & 2.19 \\
    \hline
    SKCCA & 0.8523 & 18.89 & (58.3, 57.7) & 0.5717 & 0.2707 & 76 & (0.1, 0.01) & 1.75e+2 \\
    \bottomrule
  \end{tabular}}
\end{table}

In \figurename~\ref{CBIRImages}, we present AROC of KCCA, RKCCA and SKCCA using different number of projections ($l$) in both
experiments. As visible in \figurename~\ref{CBIRImages}, the AROC of all approaches gradually increases when more projections are
used for retrieval. In addition, we observe that RKCCA achieves the largest AROC for any $l$ in both experiments, which verifies
its ability of generalizing KCCA. On the other hand, the AROC of SKCCA is at first smaller than and then exceeds that of KCCA.
This indicates that the ability of generalization of SKCCA is weaker than RKCCA, which may be attributed to the high sparsity of
dual projections computed by SKCCA as shown in \tablename~\ref{CBIR}.

\subfiglabelskip=0pt
\begin{figure}[!htbp]
\centering %
\subfigure[][]{\includegraphics[width=0.45\textwidth]{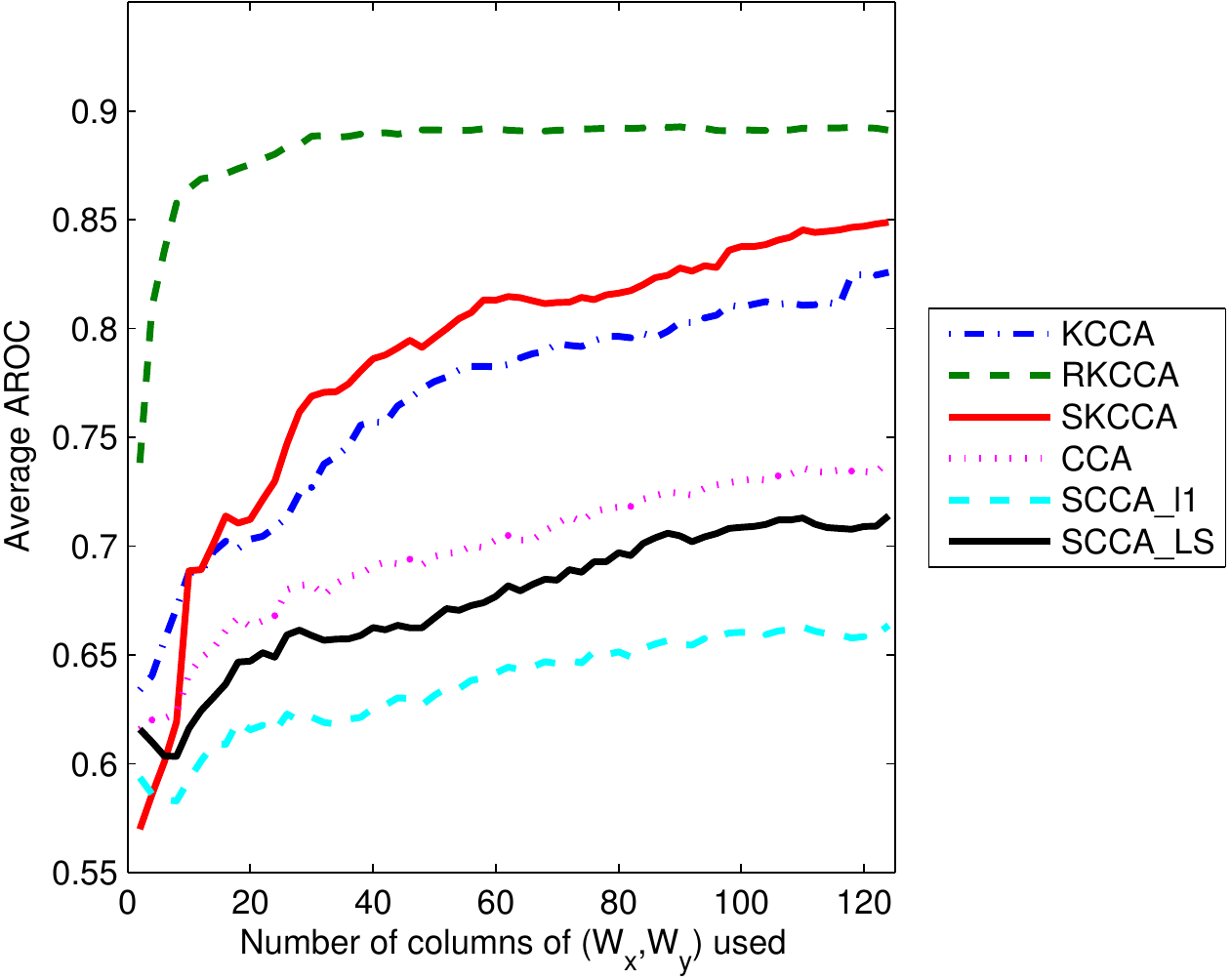} \label{GroundTruth200}} \hfil
\subfigure[][]{\includegraphics[width=0.45\textwidth]{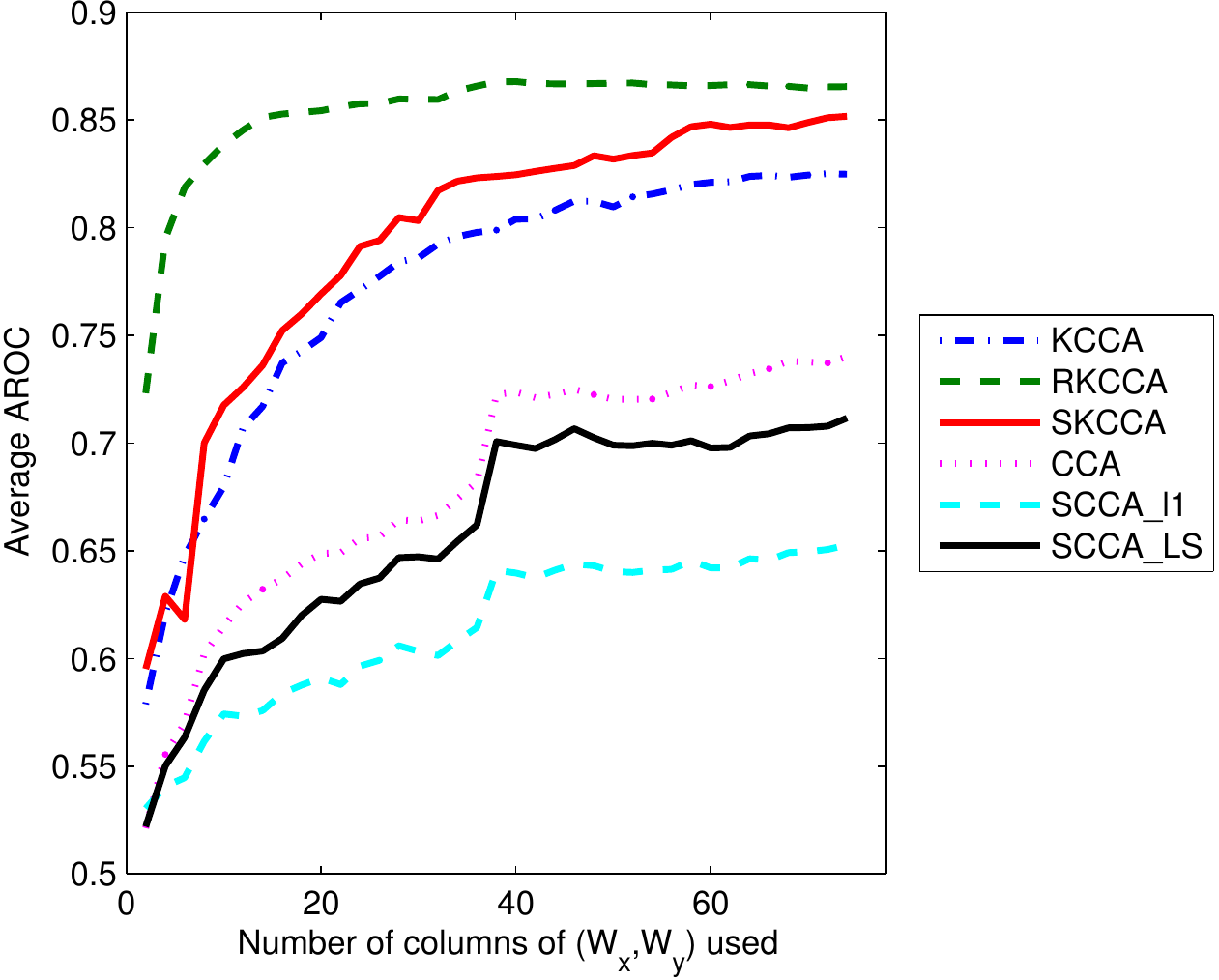} \label{SIMPLIcity1000}} %
\caption{Content-based image retrieval using KCCA, RKCCA and SKCCA:
\subref{GroundTruth200} UW ground truth data with 217 training data, \subref{SIMPLIcity1000} SIMPLIcity data with 400 training data.}%
\label{CBIRImages}%
\end{figure}

\section{Conclusions}\label{conclusions}

In this paper, we proposed a novel sparse kernel CCA algorithm called SKCCA. This algorithm is based on a relationship between kernel CCA and least squares problems which is an extension of a similar relationship between CCA and least squares problems. We incorporated sparsity into kernel CCA by penalizing the $\ell_{1}$-norm of dual vectors. The resulting $\ell_{1}$-regularized minimization problems were solved by a fixed-point continuation (FPC) algorithm. Empirical results show that SKCCA not only performs well in computing sparse dual transformations, but also alleviates the over-fitting problem of kernel CCA.

Several interesting questions and extensions of sparse kernel CCA remain. In many applications such as genomic data analysis, CCA is often performed on more than two datasets. It will be helpful to extend sparse kernel CCA to deal with multiple datasets. In the derivation of SKCCA, we did not discuss the choice of kernel function. However, it is believed that the performance of kernel CCA depends on the choice of the kernel. As for future research, we plan to study the problem of finding optimal kernel of kernel CCA for different applications. Moreover, we also plan to generalize the idea of sparse kernel CCA in this paper to involve multiple kernels.

\bibliographystyle{abbrv}
\bibliography{SparseKCCA}

\begin{thebibliography}{10}

\bibitem{Akaho01}
S.~Akaho.
\newblock A kernel method for canonical correlation analysis.
\newblock In {\em Proceedings of the International Meeting of the Psychometric
  Society}, 2001.

\bibitem{Bach03}
F.~R. Bach and M.~I. Jordan.
\newblock Kernel independent component analysis.
\newblock {\em Journal of Machine Learning Research}, 3:1--48, 2003.

\bibitem{Beck09}
A.~Beck and M.~Teboulle.
\newblock A fast iterative shrinkage-thresholding algorithm for linear inverse
  problems.
\newblock {\em SIAM Journal on Imaging Sciences}, 2(1):183--202, 2009.

\bibitem{Bie05}
T.~D. Bie, N.~Cristianini, and R.~Rosipal.
\newblock Eigenproblems in pattern recognition.
\newblock In {\em Handbook of Geometric Computing: Applications in Pattern
  Recognition, Computer Vision, Neuralcomputing, and Robotics}, pages 129--170.
  Springer, 2005.

\bibitem{Bishop06}
C.~M. Bishop.
\newblock {\em Pattern Recognition and Machine Learning}.
\newblock Springer-Verlag, 2006.

\bibitem{Bjorck73}
{\AA}.~Bj{\"o}rck and G.~H. Golub.
\newblock Numerical methods for computing angles between linear subspaces.
\newblock {\em Mathematics of Computation}, 27(123):579--594, 1973.

\bibitem{Burges98}
C.~J.~C. Burges.
\newblock A tutorial on support vector machines for pattern recognition.
\newblock {\em Data Mining and Knowledge Discovery}, 2:121--167, 1998.

\bibitem{ChuTH11}
D.~Chu, S.~T. Goh, and Y.~S. Hung.
\newblock Characterization of all solutions for undersampled uncorrelated
  linear discriminant analysis problems.
\newblock {\em SIAM Journal on Matrix Analysis Applications}, 32(3):820--844,
  2011.

\bibitem{ChuLNZ13}
D.~Chu, L.-Z. Liao, M.~K. Ng, and X.~Zhang.
\newblock Sparse canonical correlation analysis: New formulation and algorithm.
\newblock {\em IEEE Transactions on Pattern Analysis and Machine Intelligence},
  35(12):3050--3065, 2013.

\bibitem{Chu13}
D.~Chu, L.-Z. Liao, M.~K. Ng, and X.~Zhang.
\newblock Sparse kernel canonical correlation analysis.
\newblock In {\em International MultiConference of Engineers and Computer
  Scientists}, 2013.

\bibitem{Dauxois98}
J.~Dauxois and G.~M. Nkiet.
\newblock Nonlinear canonical analysis and independence tests.
\newblock {\em The Annals of Statistics}, 26(4):1254--1278, 1998.

\bibitem{Dettling04}
M.~Dettling.
\newblock Bagboosting for tumor classification with gene expression data.
\newblock {\em Bioinformatics}, 20(18):3583--3593, 2004.

\bibitem{Dhanjal08}
C.~Dhanjal.
\newblock {\em Sparse kernel feature extraction}.
\newblock PhD thesis, University of Southampton, 2008.

\bibitem{DhanjalGunn08}
C.~Dhanjal, S.~Gunn, and J.~Shawe-Taylor.
\newblock Efficient sparse kernel feature extraction based on partial least
  squares.
\newblock {\em IEEE Transactions on Pattern Analysis and Machine Intelligence},
  99(1):1347--1361, 2008.

\bibitem{Figueiredo07}
M.~A.~T. Figueiredo, R.~D. Nowak, and S.~J. Wright.
\newblock Gradient projection for sparse reconstruction: Application to
  compressed sensing and other inverse problems.
\newblock {\em IEEE Journal of Selected Topics in Signal Processing},
  1(4):586--597, 2007.

\bibitem{Fukumizu07}
K.~Fukumizu, F.~R. Bach, and A.~Gretton.
\newblock Statistical consistency of kernel canonical correlation analysis.
\newblock {\em Journal of Machine Learning Research}, 8:361--383, 2007.

\bibitem{Fyfe00}
C.~Fyfe and P.~L. Lai.
\newblock \textsc{ICA} using kernel canonical correlation analysis.
\newblock In {\em Proceedings of the Workshop on Independent Component Analysis
  and Blind Signal Separation}, pages 279--284, 2000.

\bibitem{Germann01}
U.~Germann.
\newblock Aligned hansards of the 36th parliament of canada.
\newblock \url{http://www.isi.edu/natural-language/download/hansard/}, 2001.

\bibitem{Golub96}
G.~H. Golub and C.~F.~V. Loan.
\newblock {\em Matrix Computations}.
\newblock The Johns Hopkins University Press, 3rd edition, 1996.

\bibitem{Golub94}
G.~H. Golub and H.~Zha.
\newblock Perturbation analysis of the canonical correlations of matrix pairs.
\newblock {\em Linear Algebra and Its Applications}, 210:3--28, 1994.

\bibitem{Hale08}
E.~T. Hale, W.~Yin, and Y.~Zhang.
\newblock Fixed-point continuation for $\ell_{1}$-minimization: Methodology and
  covergence.
\newblock {\em SIAM Journal on Optimization}, 19(3):1107--1130, 2008.

\bibitem{Hale10}
E.~T. Hale, W.~Yin, and Y.~Zhang.
\newblock Fixed-point continuation applied to compressed sensing:
  Implementation and numerical experiments.
\newblock {\em Journal of Computational Mathematics}, 28(2):170--194, 2010.

\bibitem{Hardoon11}
D.~R. Hardoon and J.~Shawe-Tayler.
\newblock Sparse canonical correlation analysis.
\newblock {\em Machine Learning}, 83(3):331--353, 2011.

\bibitem{Hardoon09}
D.~R. Hardoon and J.~Shawe-Taylor.
\newblock Convergence analysis of kernel canonical correlation analysis: Theory
  and practice.
\newblock {\em Machine Learning}, 74:23--38, 2009.

\bibitem{Hardoon04}
D.~R. Hardoon, S.~R. Szedmak, and J.~R. Shawe-Taylor.
\newblock Canonical correlation analysis: an overview with application to
  learning methods.
\newblock {\em Neural Computation}, 16(12):2639--2664, 2004.

\bibitem{Hotelling36}
H.~Hotelling.
\newblock Relations between two sets of variables.
\newblock {\em Biometrika}, 28:321--377, 1936.

\bibitem{JinYHL01}
Z.~Jin, J.~Y. Yang, Z.~S. Hu, and Z.~Lou.
\newblock Face recognition based on the uncorrelated discriminant
  transformation.
\newblock {\em Pattern Recognition}, 34:1405--1416, 2001.

\bibitem{Koehn05}
P.~Koehn.
\newblock {Europarl: A Parallel Corpus for Statistical Machine Translation}.
\newblock In {\em {Proceedings: the tenth Machine Translation Summit}}, pages
  79--86, 2005.

\bibitem{Kuss03}
M.~Kuss and T.~Graepel.
\newblock The geometry of kernel canonical correlation analysis.
\newblock Technical report, Max Plank Institute for Biological Cybernetics,
  Germany, 2003.

\bibitem{Lai01}
P.~Lai and C.~Fyfe.
\newblock Kernel and nonlinear canonical correlation analysis.
\newblock {\em International Journal of Neural Systems}, 10:365--374, 2001.

\bibitem{LeeHK05}
K.-C. Lee, J.~Ho, and D.~J. Kriegman.
\newblock Acquiring linear subspaces for face recognition under variable
  lighting.
\newblock {\em IEEE Transactions on Pattern Analysis and Machine Intelligence},
  27(5), 2005.

\bibitem{Melzer01}
T.~Melzer, M.~Reiter, and H.~Bischof.
\newblock Nonlinear feature extraction using generalized canonical correlation
  analysis.
\newblock In {\em Proceedings of the International Conference on Artificial
  Neural Networks}, pages 353--360, 2001.

\bibitem{Mika99}
S.~Mika, G.~R{\"a}tsch, J.~Weston, B.~Sch{\"o}lkopf, and K.-R. M{\"u}ller.
\newblock Fisher discriminant analysis with kernels.
\newblock In {\em Neural Networks for Signal Processing IX}, pages 41--48.
  IEEE, 1999.

\bibitem{Momma03}
M.~Momma and K.~P. Bennett.
\newblock Sparse kernel partial least squares regression.
\newblock In B.~Sch{\"o}lkopf and M.~K. Warmuth, editors, {\em Proceedings of
  $16^{th}$ International Conference on Computational Learning Theory}, pages
  216--230, 2003.

\bibitem{Parkhomenko09}
E.~Parkhomenko, D.~Tritchler, and J.~Beyene.
\newblock Sparse canonical correlation analysis with application to genomic
  data integration.
\newblock {\em Statistical Applications in Genetics and Molecular Biology}, 8,
  2009.
\newblock Issue 1, Article 1.

\bibitem{Rosipal01}
R.~Rosipal and L.~J. Trejo.
\newblock Kernel partial least squares regression in reproducing kernel
  \textsc{H}ilbert space.
\newblock {\em Journal of Machine Learning Research}, 2:97--123, 2001.

\bibitem{Scholkopf02}
B.~Sch\"{o}lkopf and A.~J. Smola.
\newblock {\em Learning with kernels : support vector machines, regularization,
  optimization, and beyond}.
\newblock MIT Press, 2002.

\bibitem{Scholkopf98}
B.~Sch\"{o}lkopf, A.~J. Smola, and K.-R. M\"{u}ller.
\newblock Nonlinear component analysis as a kernel eigenvalue problem.
\newblock {\em Neural Computation}, 10:1299--1319, 1998.

\bibitem{Shawe-Taylor04}
J.~Shawe-Taylor and N.~Cristianini.
\newblock {\em Kernel Methods for Pattern Analysis}.
\newblock Cambridge University Press, 2004.

\bibitem{Sriperumbudur07}
B.~K. Sriperumbudur, D.~A. Torres, and G.~R.~G. Lanckriet.
\newblock Sparse eigen methods by d.c. programming.
\newblock In {\em The 24th International Conference on Machine Learning}, pages
  831--838, 2007.

\bibitem{Sriperumbudur11}
B.~K. Sriperumbudur, D.~A. Torres, and G.~R.~G. Lanckriet.
\newblock A majorization-minimization approach to the sparse generalized
  eigenvalue problem.
\newblock {\em Machine Learning}, 85(1-2):3--39, 2011.

\bibitem{SunJY11}
L.~Sun, S.~Ji, and J.~Ye.
\newblock Canonical correlation analysis for multilabel classification: A
  least-squares formulation, extensions, and analysis.
\newblock {\em IEEE Transactions on Pattern Analysis and Machine Intelligence},
  33(1):194--200, 2011.

\bibitem{TanFyfe01}
L.~Tan and C.~Fyfe.
\newblock Sparse kernel canonical correlation analysis.
\newblock In {\em Proceedings of $9^{th}$ European Symposium on Artificial
  Neural Networks}, pages 335--340, 2001.

\bibitem{Tibshirani96}
R.~Tibshirani.
\newblock Regression shrinkage and selection via the lasso.
\newblock {\em Journal of the Royal Statistical Society (Series B)},
  58:267--288, 1996.

\bibitem{Tipping01}
M.~E. Tipping.
\newblock Sparse bayesian learning and the relevance vector machine.
\newblock {\em Journal of Machine Learning Research}, 1, 2001.

\bibitem{Vert03}
J.-P. Vert and M.~Kanehisa.
\newblock Graph-driven features extraction from microarray data using diffusion
  kernels and kernel cca.
\newblock In S.~Becker, S.~Thrun, and K.~Obermayer, editors, {\em Advances in
  neural information processing systems, volume 15}. MIT Press, 2003.

\bibitem{Vinokourov03}
A.~Vinokourov, J.~Shawe-taylor, and N.~Cristianini.
\newblock Inferring a semantic representation of text via cross-language
  correlation analysis.
\newblock In S.~Becker, S.~Thrun, and K.~Obermayer, editors, {\em Advances in
  neural information processing systems, volume 15}, pages 1473--1480. MIT
  Press, 2003.

\bibitem{Waaijenborg08}
S.~Waaijenborg, P.~C.~V. de~Witt~Hamer, and A.~H. Zwinderman.
\newblock Quantifying the association between gene expressions and dna-markers
  by penalized canonical correlation analysis.
\newblock {\em Statistical Applications in Genetics and Molecular Biology}, 7,
  2008.
\newblock Issue 1, Article 3.

\bibitem{Wahba99}
G.~Wahba.
\newblock {\em Support vector machines, reproducing kernel \textsc{H}ilbert
  spaces, and randomized \textsc{GACV}}, pages 69--88.
\newblock Advances in kernel methods $-$ Support Vector Learning. MIT Press,
  1999.

\bibitem{Wiesel08}
A.~Wiesel, M.~Kliger, and A.~O. Hero.
\newblock A greedy approach to sparse canonical correlation analysis.
\newblock 2008.
\newblock Available at http://arxiv.org/abs/0801.2748.

\bibitem{Witten09}
D.~M. Witten and R.~Tibshirani.
\newblock Extensions of sparse canonical correlation analysis with applications
  to genomic data.
\newblock {\em Statistical Applications in Genetics and Molecular Biology}, 8,
  2009.
\newblock Issue 1, Article 28.

\bibitem{WittenHastie09}
D.~M. Witten, R.~Tibshirani, and T.~Hastie.
\newblock A penalized matrix decomposition, with applications to sparse
  principal components and canonical correlation analysis.
\newblock {\em Biostatistics}, 10(3):515--534, 2009.

\bibitem{Yamanishi03}
Y.~Yamanishi, J.~P. Vert, A.~Nakaya, and M.~Kanehisa.
\newblock Extraction of correlated gene clusters from multiple genomic data by
  generalized kernel canonical correlation analysis.
\newblock {\em Bioinformatics}, 19(Suppl 1):i323--i330, 2003.

\bibitem{YangZ11}
J.~Yang and Y.~Zhang.
\newblock Alternating direction algorithms for $\ell_{1}$-problems in
  compressive sensing.
\newblock {\em SIAM Journal on Scientific Computing}, 33(1):250--278, 2011.

\bibitem{Ye05}
J.~Ye.
\newblock Characterization of a family of algorithms for generalized
  discriminant analysis on undersampled problems.
\newblock {\em Journal of Machine Learning Research}, 6:483--502, 2005.

\bibitem{ZhangThesis13}
X.~Zhang.
\newblock {\em Sparse Dimensionality Reduction Methods: Algorithms and
  Applications}.
\newblock PhD thesis, National University of Singapore, 2013.

\end{thebibliography}

\end{document}